\documentclass[final,5p,times,twocolumn]{elsarticle}
\usepackage[utf8]{inputenc}
\usepackage{newfloat}
\usepackage{caption}
\usepackage{hyperref,breakurl}
\DeclareFloatingEnvironment[fileext=cmh,placement={!ht},name=List]{myfloat}
\usepackage{ifpdf}
\usepackage{amsfonts}
\usepackage{amssymb}
\usepackage[cmex10]{amsmath}
\usepackage{multirow}
\usepackage{tikz}

\usepackage{enumitem}

\usepackage{url}
\usetikzlibrary{decorations.markings,arrows}
\usetikzlibrary{decorations.shapes}
\tikzset{decorate sep/.style 2 args=
 {decorate,decoration={shape backgrounds,shape=circle,shape size=#1,shape sep=#2}}}
\usetikzlibrary{plotmarks}
\usetikzlibrary{shadings}
\usepackage{graphicx}
\usepackage{float}
\usepackage{amscd,amsgen,amsfonts,amsbsy}
\usepackage{mathrsfs}
\usepackage{color}
\usepackage{textcomp}
\usepackage{latexsym,graphicx}
\usepackage[british]{babel}
\usepackage[font=small,labelfont=bf]{caption}
\usepackage{pgf}
\usepackage{multirow}

\usepackage{algorithm}
\usepackage[noend]{algpseudocode}
\algnewcommand{\algorithmicgoto}{\textbf{go to}}%
\algnewcommand{\Goto}[1]{\algorithmicgoto~\ref{#1}}%
\algnewcommand{\algoand}{\textbf{and }}
\algnewcommand{\algoor}{\textbf{or }}

\usepackage{subcaption}
\usepackage{footmisc}

\usepackage{epstopdf}
\newenvironment{proof}{\noindent{\bf Proof:}}{$\hfill \Box$ \vspace{10pt}}
\newtheorem{prop}{Proposition}

\newtheorem{remark}{Remark}
\newtheorem{theorem}{Theorem}
\newtheorem{lemma}{Lemma}

\newtheorem{claim}{Claim}

\DeclareMathOperator*{\argmax}{arg\,max}

\newcommand{\reg}{$^{\tiny\mbox{\textregistered}}\ $}
\DeclareMathOperator{\sign}{sign}

\begin{document}
\begin{frontmatter}

\title{\LARGE \bf
{Reconfigurable formations of quadrotors on Lissajous curves for surveillance applications}}

\author[label1]{Aseem V. Borkar\corref{cor1}}
\ead{aseem@sc.iitb.ac.in}

\author[label1]{Swaroop Hangal\corref{cor2}}
\ead{swaroop.hangal@iitb.ac.in}
\author[label2]{Hemendra Arya}
\ead{arya@aero.iitb.ac.in}
\author[label1]{Arpita Sinha}
\ead{asinha@sc.iitb.ac.in}
\author[label1]{Leena Vachhani}
\ead{leena@iitb.ac.in}

\cortext[cor1]{Corresponding author}
\cortext[cor2]{Now with General Aviation, Bengaluru, India.}
\address[label1]{Interdisciplinary Programme in Systems and Control Engineering, Indian Institute of Technology Bombay, Mumbai 400076, India}
\address[label2]{Department of Aerospace Engineering, Indian Institute of Technology Bombay, Mumbai 400076, India}

\begin{abstract}
  This paper proposes  trajectory planning strategies for  online reconfiguration of a multi-agent formation on a Lissajous curve. In our earlier work \cite{lisiros}, a multi-agent formation with constant parametric speed was proposed in order to address multiple  objectives such as repeated collision-free surveillance and guaranteed sensor coverage of the area with ability for rogue target detection and trapping.  This work addresses the issue of formation reconfiguration within this context. In particular, smooth parametric trajectories are designed for the purpose using calculus of variations. These trajectories have been employed in conjunction with a simple local cooperation scheme so as to achieve collision-free reconfiguration between different Lissajous curves. A detailed theoretical analysis of the proposed scheme  is provided. These surveillance and  reconfiguration strategies have also been validated through simulations  in MATLAB\reg for agents performing parametric motion along the curves, and by Software-In-The-Loop simulation for quadrotors. In addition, they are validated  experimentally with a team of quadrotors flying in a motion capture environment. 
\end{abstract}

\end{frontmatter}

\section{Introduction }
\label{sec_intro}
The autonomous area surveillance task  involves planning paths for a single/multiple autonomous agents with a limited sensing range. This could be done so as to  ensure that all points in an area of interest are viewed /sensed repeatedly in finite time, with guaranteed  detection of a target of interest in the area being monitored.  Multi-agent systems offer several advantages over single agent systems for repeated coverage and target detection tasks. The most significant the search and detection of a rogue element, and  reduction of the time required for a single agent, which is alleviated by the parallelism implicit in multi-agent implementations. Another notable advantage is the increased robustness due to redundancy. Some major application domains  are searching for threats \cite{targetsearch}, surveillance (\cite{cooppatrolling}, \cite{coopsurviellance}), and so on.
Some preliminary results in regard to the proposed strategy were presented in
\cite{lisiros}, where a trajectory plan was proposed for a multi-agent formation on a Lissajous curve in order to achieve the following objectives simultaneously:\vspace{-0.25cm}
\begin{enumerate}[label=  O\arabic*]\it
\item Complete and periodic coverage of the rectangular region of interest. \vspace{-0.25cm}\label{ob_1}
\item Collision-free patrolling for  agents having finite non-zero size with  agent speed bounded above by $V_{max}$.\vspace{-0.25cm}\label{ob_2}
\item Finite time entrapment and detection of a rogue element held at the center of the region.\label{ob_3}
\end{enumerate}
The prior work in \cite{lisiros} considered a non-cooperating group of agents on a Lissajous curve and exploited the geometric properties of the curve to meet the above objectives simultaneously. Also, in \cite{lisiros} a sufficient upper bound on the agent size was derived in order to guarantee collision-free motion of the multi-agent formation. 

In this paper, we extend the work of \cite{lisiros} to a cooperating reconfigurable multi-agent formation on Lissajous curves. Our proposed reconfigurable multi-agent formation guarantees smooth and collision-free trajectories for formation reconfiguration, considering the following  operations:
 $$\mbox{ 1) Agent addition, 2) Agent removal, 3) Agent replacement}.$$
 The key features of the proposed reconfiguration strategy are: \vspace{-0.25cm}
\begin{enumerate}
\item A  connected communication graph between the formation agents  having limited communication range for cooperation.\vspace{-0.25cm}
\item Parametric trajectories for smooth acceleration, deceleration and transitions between Lissajous curves.\vspace{-0.25cm}
\item Cooperative  assignment schemes for each reconfiguration operation to ensure  collision-free multi-agent transitions from one Lissajous curve to another.\vspace{-0.25cm}
\item Collision-free trajectories are  designed for agents having finite non-zero size satisfying the size bound derived in \cite{lisiros}. 
\end{enumerate}

The proposed strategy has potential applications to target search, repeated surveillance, monitoring and mapping of dynamic environments, area sweeping for cleaning, spraying, etc.

 The rest of this paper  is organized as follows:  Section \ref{litsurvey} recalls some related works. Section \ref{sec_prelims} discusses some preliminaries about Lissajous curves and the design of smooth trajectories using the calculus of variations approach. Section \ref{sec_surveillance} discusses the theoretical details of the proposed surveillance strategy  and gives a systematic algorithm to implement it, recalling some earlier work from \cite{lisiros}.  Section \ref{sec_recon} discusses the proposed formation reconfiguration strategies for addition, removal and replacement  of agents in the multi-agent formation proposed for the  surveillance strategy discussed in Section \ref{sec_surveillance}. Section \ref{sec_sim_expt} validates the proposed surveillance and reconfiguration strategies through simulations and experiment.  Section \ref{sec_conc} concludes the paper.

\section{Literature survey}\label{litsurvey}

The novelty of the work presented here is in simultaneously addressing multiple  surveillance objectives, while ensuring collision-free paths for the multiple agents having a finite non-zero size. In literature, several methods have been proposed  to address the aerial surveillance problem \cite{Leahy,Keller,Capitan,Saska,Hosseini}. \cite{Leahy} takes an approach based on temporal logic.  \cite{Keller} proposes a strategy based on parametrized curves using splines. In contrast, \cite{Capitan} takes a `planning' viewpoint based upon casting the problem as a partially observed Markov decision process (POMDP). \cite{Saska} have addressed the surveillance problem with a particle swarm optimization based approach. The article \cite{Hosseini} takes into account energy considerations leading to a distinctive estimation/optimization problem. Our proposed surveillance strategy is based on tracking a parametric curve, viz., the Lissajous curve, that lends itself to a clean analytic way of path planning with a priori guarantees in certain performance criteria.

While complete area coverage is subsumed in our objective, our scheme goes well beyond it towards surveillance and detection of rogue elements. Coverage by itself has been extensively researched as a stand-alone theme. We recall a few relevant works here. One of the approaches for single agent area coverage is cell decomposition (discussed in \cite{choset_book}) wherein an area is divided into cells and these cells are searched systematically using zig-zag scan patterns. Occupancy grid  based strategies studied in \cite{gridoffline}, use a distance transform to assign a specific numeric value to each free grid element starting from a `goal' point, and a pseudo-gradient descent approach to generate a coverage path from `start' to `goal'. The grid based algorithms such as Spiral Spanning Tree Coverage Algorithm \cite{STCpaper} and the Backtracking Spiral Algorithm \cite{BSApaper} ensure that the robot returns to its starting grid location after completion of the coverage task. Thus these algorithms can be used for repeated coverage tasks. Some of these approaches have also been extended to multi-agent scenarios (\cite{multiSTC_nonuni}, \cite{UAVmulti_polypart}). 
Another problem relevant to patrolling is target search and detection. Recent works in this direction include: gradient based strategies for multi-UAV search \cite{gradienttarget}, Voronoi partition based strategy using an uncertainty map \cite{voronoisearch}, and game theoretic search strategies \cite{gamesearch}. In comparison to these approaches, the proposed strategy gives deterministic guarantee on repeated complete coverage and target detection in finite time.

Closer in spirit to our approach are the works based on well defined geometric curves such as raster scanning zigzag paths used in cell based schemes  discussed in \cite{choset_book}, \cite{UAVmulti_polypart}, etc., and space filling curves such as the Hilbert curve for multi-agent coverage in \cite{spacefilmulti}, and for non-uniform priority based coverage in \cite{fractal} and \cite{sidharthhilbert}. These curves often involve sharp turns and require additional path planning to return to the `start' position to repeat the task. However, Lissajous curves have a simple parametric form and  an appropriate choice of describing variables results in a smooth periodic curve called a `non-degenerate Lissajous curve' (\cite{lisnondgen}) of prescribed mesh density within a rectangle of any dimensions. The parametric form of the Lissajous curve simplifies the expression for a moving reference point that can be tracked by a robot. The problem of optimal choice of a Lissajous curve for multi-agent persistent monitoring of 2-D spaces has been addressed in \cite{lissajous_multi}, where each agent is assigned a separate Lissajous curve. Unlike \cite{lissajous_multi}, the aim of this work is to explore the advantage of multiple agents following a single  curve so that each agent covers the entire area over the period of time.
This feature allows robustness against failure and fast coverage due to parallelism.
It also makes this approach suitable for surveillance tasks where each agent is equipped with different types of sensors.

In this work we also propose collision-free online reconfiguration trajectories for the  multi-agent formation on Lissajous curves proposed in \cite{lisiros}. We highlight some literature in broadly related areas with similar motivations regarding multi-agent formation  reconfiguration. \cite{vkdecentralised} have  proposed a decentralised trajectory planner which guarantees convergence of a multi-robot formation to a centralised trajectory. Additionally, they use a simple rule based assignment methodology for collision-free formation shape reconfiguration. In \cite{dandrea}, a quadrotor team is employed for building a structure. For this task, collision-free routes to shared resources, such as battery charging
stations, are computed using  reserved passing
lanes and reservation systems.

The problem of agent removal for recharging or refuelling in long endurance missions has been posed as a problem of scheduling and
goal-reassignment task in \cite{milp} and \cite{smilp}.  \cite{milp} present a heuristic method to solve a Mixed Integer Linear program (MILP) formulation in order to efficiently cover a set of targets with agent removal for recharging. In \cite{smilp}, this approach has been extended considering an energy aware optimisation objective with an initially uncertain energy expenditure model. 

In \cite{miqpvk} and \cite{miqp}, a Mixed Integer Quadratic Program (MIQP) formulation has been used to plan piecewise smooth  collision-free trajectories for formation reconfiguration. Of these, \cite{miqpvk} use integer constraints to enforce collision avoidance for formation reconfiguration of a heterogeneous team of quadrotors. In \cite{miqp} trajectories are computed for online substitution of quadrotors in  a multi-quadrotor formation.  In  \cite{milp, smilp, miqp} and \cite{miqpvk} the reported experiments have been  conducted by implementing the  proposed trajectory planners with appropriate optimisation solvers on a central computer which commands the robots in flight. 

Other recent approaches of formation reconfiguration and control include the Virtual Rigid Body abstraction in \cite{VRBmac} and the use of path planning algorithms based on sequential convex programming (SCP) in  \cite{SCPhow}. In \cite{VRBmac}, collision-free trajectories are obtained to maintain a fixed relative  quadrotor formation in a manoeuvre and also to switch between a sequence of quadrotor formations.  In \cite{SCPhow}, an incremental sequential convex programming algorithm has been proposed for  finding  feasible collision-free trajectories for quadrotor teams. These are computed in near real time on a central computer. 

Note that these works depend significantly more on centralised processing as compared to our proposed reconfiguration scheme.


\section{Preliminaries }
\label{sec_prelims}
Let the dimensions of rectangular environment be $2A\times 2B$ with $A,B \in \Bbb{R}$. We consider the Lissajous curves with parametric equation
\begin{align}
x(s(t)) &= A\cos(a s(t) ), ~ y(s(t)) = B\sin(bs(t)),
\label{eqn_lis_gen}
\end{align}
where $s$ is the parameter, $a$ and $b\in {\Bbb N}$ are  co-prime positive integer constants (having common factors,  results in the same Lissajous curve, e.g., $a:b=1:2,\ 4:8$ and $12:24$).  The coordinates $X$ and $Y$ are defined along the directions parallel to the sides of the rectangle and the origin  is chosen to be the center of the rectangular region.


For the work presented in this paper, only non-degenerate Lissajous curves (refer \cite{lisnondgen}) are considered, with the  property that its entire curve length is traversed only once along a single direction by the running parameter $s$ in the parametric period of $[0,\ 2\pi)$. To ensure non-degeneracy of  \eqref{eqn_lis_gen}, $a$  must be an odd integer. The properties of non-degenerate Lissajous curves used to derive some results in this paper are stated with proofs in the online supplement\footnote{ \url{https://drive.google.com/open?id=0B4lbdPZ-BnshS25qcVc3TlNnV1U} \label{fn1}}.
The points on the  Lissajous curve which are encountered twice within a complete traversal of the curve are called intersection points (red and green points in Fig.\ \ref{fig_lis_example}), and the points where the Lissajous curve touches the boundaries of the rectangular region of interest are called boundary points (black and magenta points in Fig.\ \ref{fig_lis_example}). Together, the intersection and boundary points are referred to as node points.
\begin{figure}[h]
\centering
\includegraphics[width=0.85\linewidth]{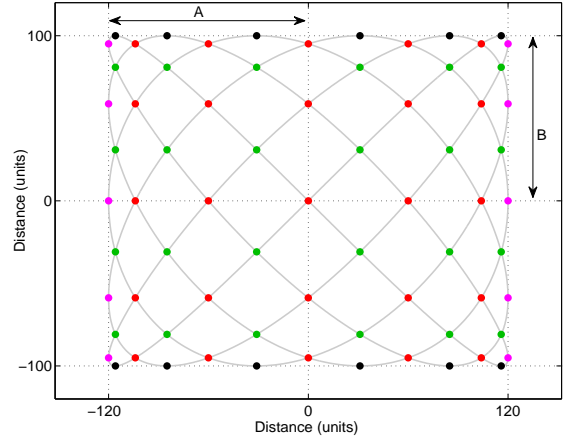}\vspace{-0.2cm}
\caption{Non-degenerate Lissajous curve with $a=5,\ b=6,\ A=120,B=100$ }
\label{fig_lis_example}
\end{figure}
 To guarantee smooth collision-free motion of the agents while ensuring that the  speed of the agents is bounded above by $V_{max}$ (the maximum allowable speed of the agents), smooth transition trajectories are used. The  trajectories  are designed to be at least twice continuously differentiable, using  calculus of variations approach.
%

For this purpose, we consider the solutions of the calculus of variations problem given in Lemma \ref{clm_calculus_variations}:
\begin{lemma}
The optimal function minimizing the integral $\int\limits_{0}^{T_f} \frac{{\dddot{g}}^2(t)}{2} dt $ for fixed end time $T_f>0$, subject to the following sets of boundary conditions, is as follows:
\begin{enumerate}[label=  C\arabic*]
\item For $g(0)=g_0$, $\dot{g}(0)=\dot{g}_0$, $\ddot{g}(0)=0$, $g(T_f)=$free, $\dot{g}(T_f)=\dot{g}_f$, $\ddot{g}(T_f)=0$, it is:\\ $g^*(t)=(\dot{g}_f-\dot{g}_0)\left(-\frac{t^4}{2T_f^3}+\frac{t^3}{T_f^2}\right)+\dot{g}_0t+g_0$.\label{bcond_c1}
\item For $g(0)=g_0$, $\dot{g}(0)=0$, $\ddot{g}(0)=0$, $g(T_f)=g_f$, $\dot{g}(T_f)=0$, $\ddot{g}(T_f)=0$, it
is:\\ $g^*(t)=(g_f-g_0)\left(10T_f^2-15T_ft+6t^2\right)\frac{t^3}{T_f^5}+g_0$. \label{bcond_c2}

\end{enumerate}\label{clm_calculus_variations}
\end{lemma}
The   proof for Lemma \ref{clm_calculus_variations} is given in the  appendix.\\

\noindent{\bf Notation :}
We define the following notations for brevity of representation: $T_0$, $T_f$ are start and end times of a parametric trajectory. $T_p=T_f-T_0$ is the time period of the transition trajectory. $\Delta t=t-T_0$ where $t$ is the current time. $\Delta g=g_f-g_0$  where $g_0,g_f$ are initial and final parametric positions at times $T_0$ and $T_f$ respectively. Similarly $\Delta \dot{g}=\dot{g}_f-\dot{g}_0$,  where  $\dot{g}_f,\ \dot{g}_0$ are initial and final parametric speeds at times  $T_0$ and $T_f$ respectively.\\

 Lemma \ref{clm_calculus_variations} gives the template for designing the following trajectories for the reconfiguration strategy:

\subsubsection*{Monotone transition trajectory}

  To design a monotone  transition trajectory  for smooth acceleration or deceleration for a parameter $g_m(t)$ from an initial  parameter value and speed to a final parameter value and speed, we use the solution for the free end state  and fixed end time boundary conditions \ref{bcond_c1} of the calculus of variations problem discussed in Lemma \ref{clm_calculus_variations}. In this case the boundary conditions on  parameter value and speed are  $g_m(T_0)=g_0$ , $\dot{g}_m(T_0)=\dot{g}_0$ and  $\dot{g}_m(T_f)=\dot{g}_f$, and the terminal value of $g_m(T_f)$  is free. Thus parameter trajectory $g_m(t)$ and its derivatives  are as follows:
\begin{align}
g_{m}(t)&=\Delta\dot{g}\left(-\frac{\Delta t^4}{2T_p^3}+\frac{\Delta t^3}{T_p^2}\right)+\dot{g}_0\Delta t+g_0,
\label{eqn_g_tr}\\
\dot{g}_{m}(t)&=\Delta\dot{g}\left(-2\frac{\Delta t^3}{T_p^3}+\frac{3\Delta t^2}{T_p^2}\right)+\dot{g}_0,
\label{eqn_g_dot_tr}\\
\ddot{g}_{m}(t)&=6\Delta\dot{g}\left(-\frac{\Delta t^2}{T_p^3}+\frac{\Delta t}{T_p^2}\right),
\label{eqn_g_2dot_tr}  \\
\dddot{g}_{m}(t)&=6\Delta\dot{g}\left(-2\frac{\Delta t}{T_p^3}+\frac{1}{T_p^2}\right).
\label{eqn_g_3dot_tr}
\end{align}
 Some properties of the $g_{m}(t)$ and its derivatives are summarised as follows:\vspace{-0.15cm}
\begin{enumerate}
\item  $\ddot{g}_{m}(T_0)=0$, $\ddot{g}_{m}(T_f)=0$ implying constant parametric speed at the beginning and the end of the  transition trajectory in the time window $[T_0,\ T_f]$.\vspace{-0.25cm}
\item From \eqref{eqn_g_2dot_tr}, $t=T_0 $ and $T_f$ are the solutions of $\ddot{g}_{m}(t)=0$, and they are the only extremizers of $\dot{g}_{m}(t)$. For $\dot{g}_0<\dot{g}_f$ from \eqref{eqn_g_3dot_tr}, $\dddot{g}_{m}(T_f)<0$ and   $\dddot{g}_{m}(T_0)>0$ which implies $\dot{g}_{m}(t)$ is monotonically increasing in $[T_0\ T_f]$ with minimum value $\dot{g}_0$ at $t=T_0$ and maximum value $\dot{g}_f$ at $t=T_f$. By similar arguments, for $\dot{g}_f<\dot{g}_0$, $\dot{g}_{m}(t)$ is monotonically decreasing in $[T_0\ T_f]$ with minimum value $\dot{g}_f$ at $t=T_f$ and maximum value $\dot{g}_0$ at $t=T_0$.
\end{enumerate}

\begin{figure}[h]
\centering
\includegraphics[width=0.9\linewidth]{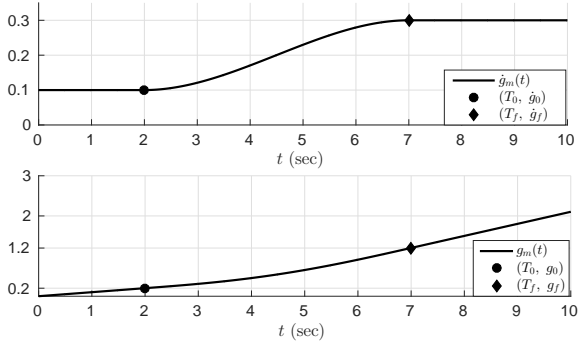}\vspace{-0.2cm}
\caption{A monotone transition trajectory example where: $T_0=2\ sec$, $g_0=0.2$, $g_f=1.2$, $\dot{g}_0=0.1/sec$, $\dot{g}_f=0.3/sec$. Thus $T_p=5\ sec$ and $T_f=7\ sec$ }
\label{fig_monotone_tr}
\end{figure}
The value of the free terminal parameter   $g_m(T_f)$ depends on the fixed end time $T_f=T_0+T_p$. Thus it depends on the width of the time interval $T_p$ specified for the transition. As a result, to obtain a desired  final parameter value   $g_m(T_f)= g_f$, the appropriate transition interval size $T_p$, and $T_f$ can be calculated by substituting $t=T_f$ (i.e., $\Delta t =T_f-T_0$) in \eqref{eqn_g_tr} and we get
\begin{align}
T_p=\frac{2(g_f-g_0)}{\dot{g}_f+\dot{g}_0}\Rightarrow T_f=T_0+\frac{2(g_f-g_0)}{\dot{g}_f+\dot{g}_0}.
\label{eqn_Tf_gtr}
\end{align}

An example of this class of  trajectories is shown in Fig.\ \ref{fig_monotone_tr}. The parameters that completely specify such a monotone trajectory are
\begin{align}
 T_0,\ g_0,\ g_f,\ \dot{g}_0 \mbox{ and } \dot{g}_f.
 \label{eqn_mono_traj_spec}
\end{align}




\subsubsection*{Symmetric transition trajectory}
In order to move a parameter $g_s$ from the initial boundary conditions $g_s(T_0)=g_0$, $\dot{g}_s(T_0)=0$ and $\ddot{g}_s(T_0)=0$,  to the  final boundary condition $g_s(T_f)=g_f$, $\dot{g}_s(T_f)=0$, and $\ddot{g}_s(T_f)=0$ (i.e., parameter is stationary at initial and final time) over a fixed time window $[T_0,\ T_f]$, we design a smooth trajectory that accelerates and decelerates symmetrically. This is done using the  fixed end state and fixed end time boundary conditions \ref{bcond_c2} of the calculus of variations problem discussed in Lemma \ref{clm_calculus_variations}.  Thus the  parameter trajectory $g_s(t)$ and its derivatives  are as follows:
\begin{align}
g_{s}(t)&=\Delta g\left(10T_p^2-15T_p\Delta t+6\Delta t^2\right)\frac{\Delta t^3}{T_p^5}+g_0,
\label{eqn_g_s_tr}\\
\dot{g}_{s}(t)&=\frac{30\Delta g}{T_p^5}\Delta t^2(T_p - \Delta t)^2,
\label{eqn_g_s_dot_tr}\\
\ddot{g}_{s}(t)&=\frac{60\Delta g}{T_p^5}\Delta t(T_p^2 - 3T_p\Delta t + 2\Delta t^2),\
\label{eqn_g_s_2dot_tr}  \\
\dddot{g}_{s}(t)&=\frac{60\Delta g}{T_p^5}(T_p^2 - 6T_p\Delta t + 6\Delta t^2).
\label{eqn_g_s_3dot_tr} 
%
%
\end{align}
Some properties of the $g_{s}(t)$ and its derivatives are summarised as follows:
\vspace{-0.15cm}
\begin{enumerate}
\item $\dot{g}_{s}(T_0)=0$, $\ddot{g}_{s}(T_0)=0$, $\dot{g}_{s}(T_f)=0$ and $\ddot{g}_{s}(T_f)=0$.\vspace{-0.25cm}
\item From \eqref{eqn_g_s_2dot_tr}, $t=T_0 , T_0 + \frac{T_p}{2}$ and $T_f$ are the solutions of $\ddot{g}_{s}(t)=0$, and are the only extremizers of $\dot{g}_{s}(t)$. For $g_0<g_f$, from \eqref{eqn_g_s_3dot_tr}, $\dddot{g}_{s}(T_0)=\dddot{g}_{s}(T_f)=\frac{60\Delta g}{T_p^3}>0$ and   $\dddot{g}_{s}(T_0+\frac{T_p}{2})=  -\frac{30\Delta g}{T_p^3} <0$ which implies $\dot{g}_{s}(t)$ attains maximum value at $t=T_0+\frac{T_p}{2}$, and  minimum value at $t=T_0$ and $t=T_f$. By similar arguments, for $g_f<g_0$, $\dot{g}_{s}(t)$ attains its minimum value at $t=T_0+\frac{T_p}{2}$ and maximum value at $t=T_0$ and $t=T_f$.\vspace{-0.25cm}
\item For  $\dot{g}_{s}(t)$, the maximum value $\dot{g}_{max}$ for the case $g_0<g_f$  and  the minimum value $\dot{g}_{min}$ for case $g_f<g_0$ are both attained at $t=T_o+\frac{T_p}{2}$ and
\begin{align}
\dot{g}_{max}=\frac{15 \vert g_f - g_0 \vert }{8T_p},\  \dot{g}_{min}=-\frac{15 \vert  g_f - g_0 \vert }{8T_p}.
\label{eqn_sym_gdot_max}
\end{align}
\end{enumerate}
 \begin{figure}[!h]

\centering
\includegraphics[width=0.9\linewidth]{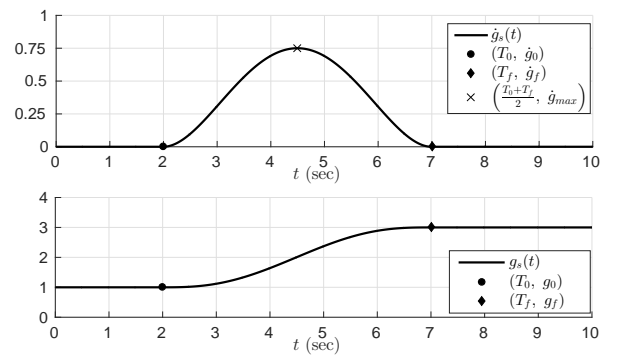}\vspace{-0.2cm}
\caption{A symmetric transition trajectory example where: $T_0=2\ sec$, $T_p=5\ sec$, $g_0=1$, $g_f=3$. Thus $T_f=7\ sec$ and $\dot{g}_{max}=0.75/sec$ at $t=\frac{T_0+T_f}{2}=4.5\ sec$}
\vspace{-0.25cm}\label{fig_symmetric_tr}
\end{figure}
An example of this class of  trajectories is shown in Fig.\ \ref{fig_symmetric_tr}. The constants which completely specify this trajectory are: 
\begin{align}
 T_0,\ T_f,\ g_0 \mbox{ and } g_f .
\label{eqn_sym_traj_spec}
\end{align}

The monotone transition trajectory is used to decelerate and accelerate the multi-agent formation along the Lissajous curve. The symmetric transition trajectory is used to design collision-free trajectories between Lissajous curves. The next section summarises and builds upon the prior work in \cite{lisiros}. 

\section{Proposed surveillance strategy }
\label{sec_surveillance}
The proposed surveillance strategy discussed in this section meets the objectives \ref{ob_1},\ref{ob_2} and \ref{ob_3} listed in Section \ref{sec_intro} (as shown in \cite{lisiros}). 
 To develop the theory for the proposed trajectory plans, we make the following assumptions:\vspace{-0.15cm}
\begin{enumerate}
\item Agents are helicopter or quadrotor like agents capable of hovering.\vspace{-0.25cm}
\item The search area is an obstacle-free rectangle of dimensions $L\times H$ and all agents  are homogeneous and identical.\vspace{-0.25cm}
\item  Each agent has a circular noise-free sensor footprint of radius $r_s < \frac{1}{2}\sqrt{L^2 + H^2}$ (i.e., half the diagonal length).\vspace{-0.25cm}
\item  Position and timing information for each agent is available from an external source (e.g., visual feedback using cameras, GPS, etc.).\vspace{-0.25cm}
\item  We assume ideal communication links without any delays or packet losses.
\end{enumerate}

Given any non-degenerate Lissajous curve, the proposed multi-agent surveillance strategy defines collision-free trajectories for multiple agents on  this curve, while ensuring that the agent formation lies on an elliptical locus centered around the origin at any instant of time. This is done by initially placing $N=a+b$ agents at equi-parametric separations on the Lissajous curve with constants $a,b$ in \eqref{eqn_lis_gen}, and moving them along the curve at equal parametric speed.

We briefly recall the results of the surveillance strategy from \cite{lisiros} in the following subsections. Later in the paper, we extend this strategy to a reconfigurable formation of agents on  Lissajous curves. 

\subsection{Multi-Agent formation}
\label{sec_formation}
For the proposed placement of agents on the Lissajous curve, the initial parameter value of the agent $i'$ is $s^{i'}(0)=s^{i'}_0=\frac{2\pi(i'-1)}{a+b}$ where $i'\in\{1,2,...,a+b\}$. For the surveillance strategy, since all agents move at the same  parametric rate $\dot{s}$ on the Lissajous curve, the parameter value of the agent at any given point in time  is given by $s^{i'}(t)=s^{i'}_0 + s(t)$, where $s(t)=\int\limits_0^t\dot{s}dt$. The  position of the agent $i'$ is a function of $s(t)$. From \eqref{eqn_lis_gen},  $y_{i'}(s(t))=B\sin\left(\frac{2(i'-1)b\pi}{a+b}+bs(t)\right)$ and $x_{i'}(s(t))=A\cos\left(\frac{2(i'-1)a\pi}{a+b}+as(t)\right)$. Using the identity $\cos(\theta)=\cos(2\pi(i'-1)-\theta)$, $x_{i'}(s(t))=A\cos\left(\frac{2(i'-1)b\pi}{a+b}-as(t)\right)$.
 With a little abuse of notation, we denote the running parameter $s(t)$ by $s$ for brevity. In Claim 2 given in the online supplement\textsuperscript{\ref{fn1}}, we have shown that by appropriate renumbering of the index $i'$ by $i$, the agents can be numbered along the elliptical locus rather than the Lissajous curve, and the resulting position coordinates of the agent $i$ with this new renumbering is given by
%
\begin{align}
x_{i}(s)=A\cos\left(\tilde{\psi}^i-as\right),\  y_{i}(s)=B\sin\left(\tilde{\psi}^i+bs\right),
\label{eqn_lis_xy_param_eqn}
\end{align}
where $\tilde{\psi}^i=\frac{2\pi(i-1)}{a+b}$ for $i\in\{1,2,...,a+b\}$.

 There can be many pairs of mutually co-prime integers satisfying the relation $N=a+b$. For example, with $N=7$ the mutually co-prime  $(a,b)$ pairs satisfying $N=a+b$ are $(1,6),\  (2,5),\ (3,4),\ (4,3),\ (5,2)$ and $(6,1)$. 
In \cite{lisiros}  an algorithm was proposed for choosing the  optimal $(a,b)$ pair that maximises the size bound on the agent and minimises the area coverage time (discussed in subsection \ref{sec_algo}). This algorithm could select a degenerate $(a,b)$ pair with even $a$ and odd $b$, for example $(a,b)=(2,5)$  for $N=7$. 
In \cite{lisiros}, this issue was addressed by  swapping the value of $a$ with $b$ and $A$ with $B$, so as  to get a non-degenerate Lissajous curve. This is equivalent to a rotation of reference frame.

In this paper, we  propose an online formation  reconfiguration strategy that switches between Lissajous curves. Hence it is convenient to maintain the same reference frame across the selected Lissajous curves. Therefore, we represent the swapping of  $a$ with $b$ and $A$ with $B$, by an equivalent phase shift in the original frame as follows:

After the swap, the position coordinates of agent $i$ on the resulting non-degenerate Lissajous curve are $$(\hat{x}_i(s'), \hat{y}_i(s'))=(B\cos(\tilde{\psi}^i-bs'),\ A\sin(\tilde{\psi}^i+as')).$$
  These coordinates can be expressed in the original reference frame  by a rotation of $\frac{-\pi}{2}$, as given below:
 \begin{align*}
\left[\begin{matrix}
x_i(s') \\ y_i(s')
\end{matrix}\right]&= \left[\begin{matrix}
0 & -1  \\ 1 & 0
\end{matrix}\right]\left[\begin{matrix}
\hat{x}_i(s') \\ \hat{y}_i(s')
\end{matrix}\right]=
\left[\begin{matrix}
A\cos(\tilde{\psi^i}+\frac{\pi}{2}+as') \\ B\sin(\tilde{\psi^i}+\frac{\pi}{2}-bs')
\end{matrix}\right].
\end{align*}

By substituting a negative parameter $s=-s'$ (reversing the direction of traversal), we get $
(x_i(s) ,y_i(s))
=\left(A\cos(\tilde{\psi^i}+\frac{\pi}{2}-as),\ B\sin(\tilde{\psi^i}+\frac{\pi}{2}+bs)\right)$. Thus the general representation of the agent positions on the Lissajous curve for the proposed strategy is
 \begin{align}
 x_i(s,\psi)\vert_{\psi=\psi^i}&=A\cos\left(\psi-as\right),\nonumber\\ y_i(s,\psi)\vert_{\psi=\psi^i}&=B\sin\left(\psi+bs\right),
 \label{eqn_alt_pos_rep}
 \end{align}
where $\psi^i=\frac{2\pi(i-1)}{a+b}+\frac{o\pi}{2}$ and the offset $o$ is
\begin{align}
o=1-(a\bmod{2}).
\label{eqn_offset}
\end{align}

\subsection{Elliptical formation locus}
\label{sec_ellipse_formation}\vspace{-0.03cm}
In \cite{lisiros}, it has been shown that the agent positions given by \eqref{eqn_alt_pos_rep} lie on a conic curve given by
\begin{align}
\frac{y^2}{B^2}+\frac{x^2}{A^2} -\frac{2xy\sin\left((a+b)s\right)}{AB} = \cos^2\left((a+b)s\right),
\label{eqn_lis_ellipse}
\end{align}
which represents an elliptical locus that is always centered about the origin $(x_o,y_o)=(0,0)$.

 For different values of parameter $s$, different elliptical loci are obtained as shown in Fig.\ \ref{fig_lis_ellipse}.
For all $k\in {\Bbb N}$, by \eqref{eqn_lis_ellipse} the parameter values $s=\frac{k\pi}{(a+b)}$ result in the ellipse $\frac{x^2}{A^2}+\frac{y^2}{B^2}=1$. Similarly parameter values $s=\frac{(2k-1)\pi}{2(a+b)}$ result in a degenerate ellipse $$\frac{x^2}{A^2}+\frac{y^2}{B^2}+(-1)^{k}\frac{2xy}{AB}=0,$$ which is the straight line  $y=\frac{(-1)^{k+1}B}{A}x$ along the diagonals of the rectangular region defined by $[-A,\ A]\times[-B,\ B]$.
\begin{figure}[h]
\centering
\includegraphics[width=0.9\linewidth]{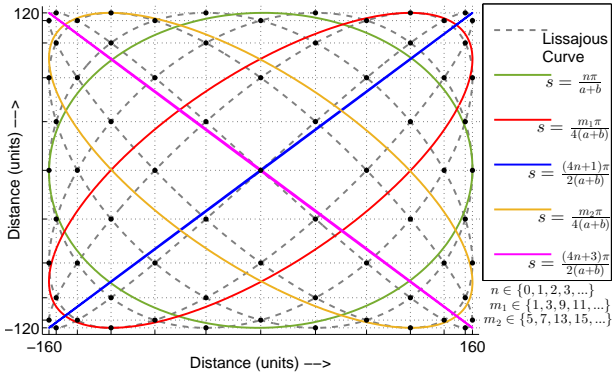}\vspace{-0.4cm}
\caption{Loci of the agent positions with placements given by \eqref{eqn_lis_xy_param_eqn} at different values of parameter $s$ for Lissajous curve having  $a=5$, $b=6$, $A=160$ and $B=120$.}
\label{fig_lis_ellipse}
\end{figure}
Notice that the equation of the ellipse given by   \eqref{eqn_lis_ellipse},  which is the locus for the multi-agent formation, is only dependent on the parameter $s$ and not $\psi^i$. Now in \eqref{eqn_alt_pos_rep}, for an agent $i$, if we make $s$ constant (fixing the ellipse) and vary the  parameter $\psi$, then  we can achieve a parametric motion along the formation ellipse. Note that the formation ellipse has a parametric length $2\pi$ in terms of the parameter $\psi$,  thus the proposed formation places the agents at equi-parametric intervals ($\tilde{\psi}^i=\frac{2\pi(i-1)}{N}$ for $i\in\{1,2,...,N\}$) along the parametric length of the formation ellipse. We use this idea later to move the agents from one Lissajous curve to another for formation reconfiguration.

\subsection{Speed profile of the agents}
\label{sec_vel}
Let $V_{max}$ be the largest permissible linear speed for each agent. Differentiating \eqref{eqn_alt_pos_rep} and assuming the running parameters $\psi$ and $s$ to be  functions of time, the components of velocity along the $x$-direction and the $y$-direction are 
\begin{align}
\dot{x}&=-A\sin(\psi-as)(\dot{\psi}-a\dot{s}),\label{eqn_xdot}\\
\dot{y}&=B\cos(\psi+bs)(\dot{\psi}+b\dot{s}),\label{eqn_ydot}
\end{align}
and the resultant speed $V$ is given by
\begin{align}
V(t)&=\sqrt{\dot{y}^2+\dot{x}^2}.
\label{eqn_V_alt_param_a2b2}
\end{align}
It is desirable to have a continuous velocity profile of the agents in order to facilitate a  practical implementation of the proposed surveillance and formation reconfiguration strategy. Furthermore, the resultant speed must be maintained below $V_{max}$. For the surveillance mission, it was proposed in \cite{lisiros} that the parameter $\psi=\psi^i=\frac{2\pi(i-1)}{N}$ be a constant and the agents move on the Lissajous curve at a constant non-zero parametric speed $\dot{s}_{nom}$ given by
\begin{align}
\dot{s}_{nom}=\frac{V_{max}}{\sqrt{A^2a^2+B^2b^2}}.
\label{eqn_sdot_choice}
\end{align}
Thus $\dot{\psi}=0$ and it was shown in \cite{lisiros} that selecting the nominal value of $\dot{s}=\dot{s}_{nom}$
guarantees  $V(t)\leq V_{max}$.

\subsection{Sensing  and communication range of agents}
\label{sec_sensor_comm}

From
\eqref{eqn_alt_pos_rep}, the $x$ and $y$ coordinate separation between any two agents $i,j$ in the proposed formation, is given by\begin{center}
 $\vert x_i-x_j \vert=2A\left\vert\sin\left(\Psi_p-as\right)\sin\left(\Psi_m\right)\right\vert\mbox{ and }$ \\ $\vert y_i-y_j \vert=2B\left\vert\cos\left(\Psi_p+bs\right)\sin\left(\Psi_m\right)\right\vert, $
\end{center}
where $\Psi_p=\frac{\psi^i+\psi^j}{2}$, and $\Psi_m=\frac{\psi^i-\psi^j}{2}$. As a result the Euclidean distance between the agents $i$ and $j$ is
 \begin{align}
D_{ij}= 2\sin(\Psi_m)\sqrt{A^2\sin^2\left(\Psi_p-as\right)+B^2\cos^2\left(\Psi_p+bs\right)}.
\label{eqn_dist_ij}
\end{align}
  For adjacent agents  along the elliptical locus, $\psi^j=\psi^i \pm\frac{2\pi}{N}$, because for the proposed strategy, the agents are equi-parametrically distributed along the elliptical locus for all time. Hence the Euclidean distance between adjacent agents is given by
\begin{align}
D_{ad}=2\sin\left(\frac{\pi}{N} \right)\sqrt{A^2\sin^2\left(\Psi_p-as\right)+B^2\cos^2\left(\Psi_p+bs\right)}.
\label{eqn_dist_ii1}
\end{align}
 $D_{ad}$ is bounded above by $D_{M}=2\sin\left(\frac{\pi}{N} \right)\sqrt{A^2+B^2}$. As a result, if each agent  has a circular sensor footprint of radius $r_s$,
%
then  overlapping sensor footprints of parametrically adjacent agents along the elliptical locus can be guaranteed at any value of the parameter $s$ by ensuring $r_{s}\geq \frac{D_{M}}{2}$. Similarly,
by ensuring that each agent's spherical communication range has radius $r_{com}>D_M$, we  guarantee that  the adjacent agents can communicate for cooperation during formation reconfiguration manoeuvres. In practice, considering the presence of curve tracing errors in the implementation using real robotic platforms such as quadrotors, the lower bounds on sensor footprint radius $r_s$ and communication range radius $r_{com}$ are chosen as
   \begin{align}
   r_{sm} &= \eta \sin\left(\frac{\pi}{N} \right)\sqrt{A^2+B^2} \mbox{ and }
\label{eqn_rs_lb}\\
 r_{cm}&= 2\eta\sin\left(\frac{\pi}{N}\right)\sqrt{A^2+B^2}
 \label{eqn_r_com}
\end{align}
respectively, where  $\eta \geq 1$ is a safety factor to ensure sufficient sensor footprint overlap and communication range.

\subsection{Agents with non-zero size and coverage time}\label{sec_size_time}
In practice, real agents such as ground robots or  quadrotors have a non-zero size and are not point agents. Thus in \cite{lisiros}, an upper bound $r_{du}$, on the radius of the circular hull was derived, which contains the physical dimensions of the agent. This upper bound is given by:
\begin{align}
r_{du}=\sin\left(\frac{\pi}{N}\right)\frac{AB}{\sqrt{A^2a^2+B^2b^2}}.
\label{eqn_rdm_ub}
\end{align}
 Agents having sizes smaller than this bound are guaranteed to have collision-free trajectories for the proposed surveillance strategy.

Since  $N$ agents (where $N=a+b$) are initially placed along the curve with a parametric separation of $\frac{2\pi}{a+b}$ and all agents move with equal parametric speed $\dot{s}_{nom}$ along the curve,  a parametric displacement of $\frac{2\pi}{a+b}$ for all the agents guarantees that the entire Lissajous curve is collectively traversed by all the agents. As a consequence, an upper bound on the time taken to collectively cover the entire rectangular area by the multi-agent formation is
 \begin{align}
 T_{cov}=\frac{2\pi}{(a+b)\dot{s}_{nom}}
 =\frac{2\pi\sqrt{A^2a^2+B^2b^2}}{NV_{max}}.
 \label{eqn_T}%
 \end{align}

The results  proved in \cite{lisiros} can be summarised  by the following theorem:

\begin{theorem}
\label{thm_iros}
Given a  non-degenerate Lissajous curve described by \eqref{eqn_lis_gen} having parameters $A,B,a,b$, the  multi-agent formation of  parametric point agents on the curve given by \eqref{eqn_lis_xy_param_eqn}, equipped with  a circular sensor footprint of radius $r_s$ given by \eqref{eqn_rs_lb} and moving along the curve with an equal parametric speed $\dot{s}_{nom}$ given by \eqref{eqn_sdot_choice} guarantees the fulfilment of objectives:\\ 
1) Collision-free paths for the agents in the formation, with  agent speed bounded above by $V_{max}$.\\
2) Complete and repeated coverage of the rectangular area by the multi-agent formation.\\
3) Finite time detection of a rogue element trying to escape from the region starting from the center of the region.
\end{theorem}
\subsection{Selection of Lissajous curve parameters} \label{sec_algo}
Suppose the number of agents to be used for the proposed surveillance  strategy is $N$, and they monitor a rectangular area of dimensions $L\times H$. Then we choose the Lissajous curve constants $A=\frac{L}{2}$ and $B=\frac{H}{2}$ in \eqref{eqn_lis_gen}. Recall that to achieve the proposed multi-agent formation the agent positions are defined by \eqref{eqn_alt_pos_rep} where $N=a+b$ for a pair of co-prime positive integers $(a,b)$. For the proposed strategy we select the $(a,b)$ pair  considering  the following claim.
\begin{claim} (\cite{lisiros})
For a given $N$, the value $a^*=\frac{B^2N}{A^2+B^2}$ with $b=N-a^*$ is the minimizer of  $T_{cov}$ and maximizer of $r_{du}$, where $T_{cov}$ and $r_{du}$ are given by  \eqref{eqn_T} and \eqref{eqn_rdm_ub} respectively.
\label{clm_T_min}
\end{claim}

\begin{algorithm}[h]
\caption{\textproc{Curve$\_$Select}}
\label{algo_curveselection}
\begin{algorithmic}[1]
\Statex {\bf Inputs:} $A$, $B$, $N$
\Statex {\bf Functions:} $GCD$
\Statex {\bf Outputs:} $a$,$b$,$o$
\State $a^*=\frac{B^2 N}{A^2+B^2}$, $d_u=\lceil a^*\rceil-a^*$, $d_l=a^*-\lfloor a^*\rfloor$
\State $c=1$
\If {$d_u\leq d_l$ \algoor $a^*<1$}
\State $k_c=\lceil a^*\rceil$, $m=0$
\EndIf
\If {$d_u> d_l$ \algoor $a^*>N-1$}
\State $\ k_c=\lfloor a^*\rfloor$, $m=1$
\EndIf
 \While {$GCD(k_c,\ N-k_c)\neq 1$ }
\State $k_c=k_c+(-1)^{c+m}c$
\State $c=c+1$
  \EndWhile
  \State $a=k_c$, $b=N-k_c$
 \State  $o=0$
\If {$GCD(k_c,\ 2)=2$}
\State $o=1$
\EndIf
\end{algorithmic}
\end{algorithm}
  In \cite{lisiros}, it has been argued that for a $\delta>0$,  $T_{cov}(a^*+\delta)=T_{cov}(a^*-\delta)$. Hence, though $a^*$ may not be an integer, we  find the positive integer $k_c$ such that it is the nearest integer to $a^*$ that yields a co-prime $(k_c,N-k_c)$ pair. Then we select $(a,b)=(k_c,N-k_c)$. This choice is the coprime integer pair that minimises the value of $T_{cov}$ and maximises $r_{du}$ while satisfying $a+b=N$. We search for this mutually co-prime pair iteratively using Algorithm \ref{algo_curveselection}. The completeness of this algorithm has been proved in \cite{lisiros}. In the worst case, Algorithm \ref{algo_curveselection} selects the following:
\begin{align}
(a,b)=\begin{cases} (1,N-1), \mbox{ if } A\geq B\\
								 (N-1,1),\mbox{ if } A<B	.
\end{cases}
\label{eqn_wrst}
\end{align}

 In case the selected $(a,b)$ pair corresponds to a degenerate Lissajous curve, then instead of the swapping  $a$ with $b$ and $A$ with $B$, as done in \cite{lisiros}, the Algorithm \ref{algo_curveselection} computes the phase offset $o$ given by \eqref{eqn_offset}.

\subsection{Number of agents in the formation}
To practically implement the proposed surveillance strategy we need to ensure the following:\vspace{-0.15cm}
\begin{enumerate}
\item Sufficient sensor footprint radius to ensure overlapping sensing ring formation (i.e., $r_s \geq r_{sm}$ in \eqref{eqn_rs_lb}).\vspace{-0.25cm}
\item Sufficient communication range between agents for cooperation  (i.e., $r_{com} \geq r_{cm}$ in \eqref{eqn_r_com}).\vspace{-0.25cm}
\end{enumerate}
To ensure this for a given sensing capability $r_s$ and communication range $r_{com}$ of a single agent, we need to compute the minimum number of agents $N_{min}$ for which the formation is defined. Thus from \eqref{eqn_rs_lb}, defining $\mathcal{R}=\eta\sqrt{A^2+B^2}$, the minimum number of agents necessary to ensure $r_s\geq r_{sm}$ is
\begin{align*}
N_s=\left\lceil \pi \left\vert\sin^{-1}\left(\frac{r_{1}}{\mathcal{R}}\right)\right\vert^{-1}\right\rceil \mbox{ with }r_1= \begin{cases}
								 r_{s},\quad \mbox{if } r_{s}<\mathcal{R}\\
								 \mathcal{R},\quad \mbox{otherwise. }  \end{cases}
\end{align*}
Similarly, from \eqref{eqn_r_com}, the minimum number of agents required to ensure
$r_{com}\geq r_{cm}$ are
\begin{align*}
N_c=\left\lceil \pi \left\vert\sin^{-1}\left(\frac{r_{2}}{2\mathcal{R}}\right)\right\vert^{-1}\right\rceil \mbox{ with }r_2= \begin{cases} r_{com}, \mbox{ if } r_{com}<2\mathcal{R}\\
								 2\mathcal{R},\mbox{ otherwise. }					\end{cases}
\end{align*}
Thus the minimum number of agents in the formation that are necessary to guarantee both $r_s\geq r_{sm}$ and $r_{com}\geq r_{cm}$, is given by
\begin{align}
N_{min}=\max \{N_s,N_c\} \geq 2.
\label{eqn_Nmin}
\end{align}
For the reconfiguration strategy, we design trajectories for addition, removal and replacement of agents in the subsequent section.   We assume that  the maximum number of extra agents $N_{extra}>0$ to be used in addition to the $N_{min}$ agents is pre-defined, and the maximum number of agents in the formation is thus calculated as
\begin{align}
N_{max}=N_{min}+N_{extra}.
\label{eqn_Nmax}
\end{align}
For normal operation we use $N$ agents where $N_{min}< N \leq N_{max}$.
\subsection{Bound on agent size}
Our objective is to design smooth trajectories for reconfiguring the multi-agent formation from one Lissajous curve to another depending on the number of agents being used. Given a number of agents $N_j$ and the corresponding Lissajous curve $(a_j,b_j)$ selected using Algorithm \ref{algo_curveselection}, from \eqref{eqn_rdm_ub} the upper bound on the circular hull radius encompassing the dimensions of the agent for the reconfigurable formation is selected  as
\begin{align}
r_{dm}=\min_{j\in S_N}\left\{\left.\frac{AB}{\sqrt{A^2 a_j^2+B^2 b_{j} ^2}} \right\vert N_j=N_{min}+j \right\}\times \sin\left(\frac{\pi}{N_{max}}\right),
\label{eqn_rdm_reconfig}
\end{align}
 where $S_N:=\{0,1,2,...,N_{extra}\}$. Thus $r_{dm}\leq r_{du_j}=\frac{AB}{\sqrt{A^2 a_j^2+B^2 b_{j} ^2}}\sin\left(\frac{\pi}{N_{l}}\right), \forall j\in S_N.$

 Thus we have Algorithm \ref{algo_initializer}, which initialises all the  parameters discussed above for all the agents.
 \begin{algorithm}[h]
\caption{\textproc{Initialisation}}
\label{algo_initializer}
\begin{algorithmic}[1]

\Statex {\bf Inputs: } $L$, $H$, $ r_s$, $r_{com}$ $V_{max}$, $N_{extra}$, $\eta$
\Statex {\bf Functions: } \textproc{Curve$\_$Select}
\Statex {\bf Outputs: } $A$, $B$, $a$, $b$, $o$, $N$, $N_{min}$, $N_{max}$, $\dot{s}_{nom}$, $r_{dm}$  and
\Statex$ (x_i(0), y_i(0))\ \forall i\in\{1,...,N\}$
\State $A=\frac{L}{2}$, $B=\frac{H}{2}$
\If {$r_s<\eta \sqrt{A^2+B^2}$}
\State $r_1=r_s$
\Else { $r_1=\eta\sqrt{A^2+B^2}$}
\EndIf
\If {$r_{com}<2\eta \sqrt{A^2+B^2}$}
\State $r_2=r_{com}$
\Else { $r_2=2\eta\sqrt{A^2+B^2}$}
\EndIf
\State $N_s=\left\lceil \pi \left\vert\sin^{-1}\left(\frac{r_{1}}{\eta\sqrt{A^2+B^2}}\right)\right\vert^{-1}\right\rceil$
\State $N_c=\left\lceil \pi \left\vert\sin^{-1}\left(\frac{r_{2}}{2\eta\sqrt{A^2+B^2}}\right)\right\vert^{-1}\right\rceil $
\State $N_{min}= \max\{N_s,N_{c}\}$, $N_{max}=N_{min}+N_{extra} $

\For {$j=0,1,...,N_{extra}$}
\State $N_j=N_{min}+j$
\State [$a_j,b_j,o_j$]=\textproc{Curve$\_$Select}($A$, $B$, $N_j$)
\State $r_{d_j}=\frac{AB}{\sqrt{A^2a_j^2+B^2b_j^2}} $
\EndFor
\State $r_{dm}=\min\left\lbrace \left. r_{d_j}\right\vert j\in\{0,...,N_{extra}\}\right\rbrace\times\sin\left(\frac{\pi}{N_{max}}\right) $
\State $N=N_1$, $a=a_1$, $b=b_1$ and $o=o_1$
\State  $\dot{s}_{nom}=\frac{V_{max}}{\sqrt{A^2a_1^2+B^2b_1^2}}$
\For {$i=1,...,N$}
\State $ \left[\begin{matrix}x_i(0)\\ y_i(0)\end{matrix}\right]= \left[\begin{matrix} A\cos\left(\frac{2\pi(i-1)}{N}+\frac{o_1\pi}{2}\right)\\ B\sin\left(\frac{2\pi(i-1)}{N}+\frac{o_1\pi}{2}\right)\end{matrix}\right]$
\EndFor

\end{algorithmic}
\end{algorithm}

\section{Formation reconfiguration strategy\vspace{-0.2cm}}
\label{sec_recon}
In this section, we extend the  proposed surveillance strategy (discussed in Section \ref{sec_surveillance}) for a more practical setting where agents may need to be added, removed or replaced from the formation on the go. This is useful for applications where the surveillance task might last for long durations. Furthermore, the trajectories designed for these tasks must guarantee collision-free motion of the agents. For this an appropriate selection of  trajectories is  done by the agents via cooperation. This is achieved by information exchange on a communication channel and this channel is established between adjacent agents  of the formation by ensuring that the communication range $r_{com}>r_{cm}$ (given by \eqref{eqn_r_com}).

 Suppose the formation initially consists  of $N$ agents and our objective is to switch to a formation of $N-1$ or $N+1$ agents. For a  reconfiguration, the trajectories are planned to move these agents from  the current Lissajous curve  corresponding to $N$ agents  to the Lissajous curve corresponding to $N-1$ or $N+1$ agents (selected in either case by Algorithm \ref{algo_curveselection}). For the replacement operation, we design a simple exchange step  where the agent to be replaced is removed and a new agent takes its place on  the same Lissajous curve. The proposed surveillance and reconfiguration strategy has been designed considering aerial agents such as helicopters and quadrotors, which are capable of safely decelerating to a zero speed in flight (hover) during operation.

 Using the parametric representation in \eqref{eqn_alt_pos_rep}, we design collision-free parametric trajectories based on cooperation for removal, addition and replacement of a single agent. For the discussions in subsequent subsections, we use the  notations given in Table \ref{tab_notation}.

\begin{table}[h]
\centering
\caption{Notations}
\label{tab_notation}
\resizebox{7.5cm}{!}{
\begin{tabular}{|l|l|}
\hline
$N_c$                & Number of agents before reconfiguration                                                                                                                                                            \\ \hline
$(a_c,b_c,o_c)$      & \begin{tabular}[c]{@{}l@{}}Lissajous curve constants for $N_c$ agents\\ given by  Algorithm \ref{algo_curveselection} with $N_c=a_c+b_c$\end{tabular}                                                                   \\ \hline
$s^i_c(t),\psi^i_c(t)$     & \begin{tabular}[c]{@{}l@{}}Curve parameters of agent $i$ in terms\\ of Lissajous curve for $N_c$ agents at time $t$\end{tabular}                                                                                       \\ \hline
$N_d$                & Number of agents after reconfiguration                                                                                                                                                             \\ \hline
$(a_d,b_d,o_d)$      & \begin{tabular}[c]{@{}l@{}}Lissajous curve constants for $N_d$ agents\\ given by Algorithm \ref{algo_curveselection} with $N_d=a_d+b_d$\end{tabular}                                                                   \\ \hline
$s^i_d(t),\psi^i_d(t)$     & \begin{tabular}[c]{@{}l@{}}Curve parameters of agent $i$ in terms\\ of Lissajous curve for $N_d$ agents at time $t$\end{tabular}                                                                                       \\ \hline
$\psi^i_D$                & \begin{tabular}[c]{@{}l@{}}$\psi_d$ parameter value corresponding to\\ formation positions on Lissajous curve for\\ $N_d$ agents assigned to agent $i$ for transition\end{tabular}                                                                                      \\ \hline
$\Delta^{i}_\psi$    & \begin{tabular}[c]{@{}l@{}}Displacement in parameter  $\psi^i_d$  for agent $i$\\ for reconfiguration to $\psi^i_D$\end{tabular}                                                                   \\ \hline
$\Delta^{ij}\psi(t)$    & \begin{tabular}[c]{@{}l@{}}$\psi$ parameter separation between\\ agent $i$ and $j$ at time $t$  \end{tabular}                                                                  \\ \hline
\end{tabular}}
\end{table}


Prior to any reconfiguration, the $N_c$ agents lie on an elliptical locus defined by the value of the parameter $s_c\in[0,\ 2\pi)$ according to \eqref{eqn_lis_ellipse}. Also  for $N_c$ agents on Lissajous curve described by $(a_c,b_c,o_c)$, the agents are equi-parametrically spaced along an ellipse in terms of parameter $\psi_c \in[0,\ 2\pi)$ (refer \eqref{eqn_alt_pos_rep}). Thus parametric separation between the adjacent agents in terms of parameter $\psi_c$  is $\frac{2\pi}{N_c}$. This fact is used by each agent to identify its adjacent agents' parameters through communication.

\subsection{Steps for formation reconfiguration}\label{sec_steps_general}
A brief outline of the steps involved in the reconfiguration operation are as follows:
\begin{itemize}
\item The multi-agent formation of $N_c$ agents decelerates to a stop on the Lissajous curve  described by constants $(a_c,b_c,o_c)$. This is done by using the monotone deceleration trajectory \eqref{eqn_g_tr} for the parameter $s_c$, where the value of $\dot{s}_c=\dot{s}_{nom}$ is decelerated smoothly to  $\dot{s}_c=0$ 
\item For addition and removal operations, where the number of agents after reconfiguration $N_d \neq N_c$, a motion of the agents along the formation ellipse is necessary to reconfigure to the destination Lissajous curve for $N_d$ agents (described by constants $(a_d,b_d,o_d)$). Thus the  parameter  transformation is done to express the curve parameters $(\psi_c,\ s_c)$ in terms of the destination Lissajous curve  as $(\psi_d,\ s_d)$, for all agents.
\item The agents are then assigned a destination position on the Lissajous curve for $N_d$ agents by a cooperative assignment scheme, and this is followed by a motion of the agents along the formation ellipse to these assigned positions by variation of parameter $\psi_d$. Since the agents require to accelerate from rest from the Lissajous curve for $N_c$ agents and decelerate back to rest on their assigned positions on the Lissajous curve for $N_d$ agents, the parameter $\psi_d$ is varied using the symmetric transition trajectory \eqref{eqn_g_s_tr}.
\item Upon reaching the Lissajous curve for $N_d$ agents, the agents now accelerate along this  new Lissajous curve to $\dot{s}_d=\dot{s}_{nom}$ using the monotone trajectory in \eqref{eqn_g_tr} for parameter $s_d$ to resume normal surveillance operation.
\end{itemize}

The common steps involved in a reconfiguration operation are discussed in further detail below: 
\subsubsection{ Monotone decelaration of $s_c$}\label{sec_mono_sc}
When any one of the three  reconfiguration operations is   initiated at time $T_R$, all the formation agents are brought to  a halt on the Lissajous curve for $N_c$ agents. 
This is done by decelerating the $\dot{s}_c$ to $0$.
One of the formation agents chosen as the reconfiguration initiator agent $i_I$, computes $\tilde{s}_f=s^{i_I}_c(T_R)+\frac{\pi}{8N_c}$. The choice of $i_I$ for each reconfiguration operation is operation specific and will be discussed later. If $\tilde{s}$ is used as the stopping $s_c$ parameter value for all active formation agents the formation locus lies on the elliptical locus given by \eqref{eqn_lis_ellipse} for $s_c=\tilde{s}_f$. 

For the replacement operation the final stopping value for agent $i$ in the formation is selected as
\begin{align}
s^i_f=s^{i_I}_{cf}= \tilde{s}_f,
\end{align} 
and this value is communicated to all formation agents by agent $i_I$ via the communication links.
\begin{figure}[h]
\centering
\includegraphics[width=0.8\linewidth]{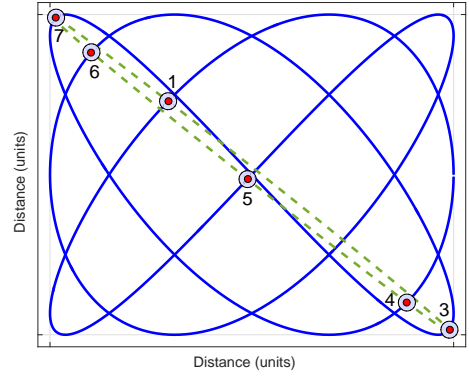}\vspace{-0.2cm}
\caption{An agent removal example where agent 2 is removed and the remaining agents cannot have a collision-free transition along the dotted green elliptical locus due to their non-zero size}
\label{fig_col_ellipse}
\end{figure}
For the addition and removal applications, we intend to reconfigure the formation by moving along this ellipse. For agents having non-zero dimensions, collisions can  occur for a narrow elliptical locus as shown in Fig.\ \ref{fig_col_ellipse}. Thus it is necessary to derive the range of parameter $s_c$ for which motion along the ellipse should be prohibited, and accordingly select an appropriate stopping parameter value $\tilde{s}_f$. To do this, we have the following proposition:

\begin{prop}
For adjacent agents, $i$ and $j$ of an $N_c$ agent formation having dimensions within a circular hull bound $r_{dm}$, and minimum $\psi$ parameter separation $\Delta^{ij}_{min}=\min_{t\in\Bbb {R}^+}\vert\psi_c^j(t)-\psi_c^i(t)\vert>0$, all elliptical loci corresponding to parameter
\begin{align}
s_c\not\in S_{avoid}
\label{eqn_s_avoid_set}
\end{align}
  are feasible for collision-free agent transitions along these loci, where
   $S_{avoid}=\cup_{k\in\Bbb N}\left(s_{diag}(k)-\delta_s ,\ s_{diag}(k)+\delta_s \right)\pmod{2\pi}$, with $s_{diag}(k)=\frac{(2k-1)\pi}{2N_c}$,  $\delta_s=\frac{\pi}{2N_c}\frac{r_{dm}\sqrt{A^2+B^2}}{AB}\left\vert\sin\left(\frac{\Delta^{ij}_{min} }{2}\right)\right\vert^{-1}$.
  \label{prop_ellipse_avoid}
\end{prop}

The proof of Proposition \ref{prop_ellipse_avoid} is given in the appendix.

\begin{remark} \label{rem_feasible_ellipse}In Proposition \ref{prop_ellipse_avoid} if the value of $\Delta^{ij}_{min}$ decreases, the $\delta_s$ increases as it is proportional to $\left\vert\sin\left(\frac{\Delta^{ij}_{min} }{2}\right)\right\vert^{-1}$. If $\delta_s\geq\frac{\pi}{2N_c}$   then no ellipse is feasible for collision-free transitions as $S_{avoid}=[0,\ 2\pi)$. Thus to have a  feasible elliptical locus for transitions we must ensure $\delta_s<\frac{\pi}{2N_c}$.
\end{remark}

 It will be shown in  later sections that for addition $\Delta^{ij}_{min}=\frac{2\pi}{N_d}$ and  for removal $\Delta^{ij}_{min}=\frac{2\pi}{N_c}$. Thus the stopping parameter value for agent $i_I$ is chosen as:
   \begin{align}
 s^{i_I}_{cf}=\begin{cases}  \begin{split} & s_{diag}(k') +\delta_s \mbox{, if }\tilde{s}_f \in S_{avoid}\mbox{ for } k=k'\\
& \tilde{s}_f \mbox{, otherwise, }
  \end{split} \end{cases}
\end{align}
where $s_{diag}$ and $S_{avoid}$ are as defined in Proposition  \ref{prop_ellipse_avoid}. The value of $s^{i_I}_{cf}$ is communicated to the remaining agents in the formation using the communication links and
  \begin{align}
 s^{i}_{cf}=s^{i_I}_{cf}.
 \label{eqn_sf_rem}
\end{align}
We assume that the transition is initiated for  all agents at the same time, i.e., $T^i_R\approx T^j_R$.
The formation agents use the monotone transition trajectory given by \eqref{eqn_g_tr}  (discussed in Section \ref{sec_prelims}) to smoothly decelerate to rest at $s^i_c=s^{i_I}_{cf}$. Thus  from \eqref{eqn_mono_traj_spec}, the constants  that characterize this trajectory for agent $i$  are:
$$T_0=T^i_R,\ \dot{g}_0= \dot{s}_{nom}, \ \dot{g}_f= 0,\ g_0= s^i_c(T_0)\mbox{ and }\ g_f=s^{i}_{cf}.$$



\subsubsection{ Parameter transformation}
\label{sec_param_tx} 
This step is carried out by each formation agent after completion of the monotone deceleration of $\dot{s}_c$ to $0$ (discussed in Section \ref{sec_mono_sc}).
 For brevity of notation, parameters  $s_c(t)$ and $\psi_c(t)$,  which are  functions of time,  are written as $s_c$ and $\psi_c$ respectively. This step is common to the agent addition and removal operations as both involve agent motion from Lissajous curve for $N_c$ agents to the Lissajous curve for $N_d$ agents along the formation ellipse, and is not necessary for agent replacement as $N_c=N_d$.  For such a motion to be possible, the agent positions on the current Lissajous curve and the assigned positions on the Lissajous curve for $N_d$ agents must lie on the same ellipse at any given instant. In other words, for $(a_c,b_c,s_c)$ and $(a_d,b_d,s_d)$, \eqref{eqn_lis_ellipse} must result in the same ellipse. This implies that $(a_c+b_c)s_c=(a_d+b_d)s_d$, or,   \begin{align}s_d=\frac{a_c+b_c}{a_d+b_d}s_c.
\label{eqn_s_transform}
\end{align}
 From \eqref{eqn_alt_pos_rep}, the position of agent $i$ on the Lissajous curve $(a_c,b_c)$ in terms of parameters $\psi^i_c$ and $s^i_c$  is:
$$(x_i,y_i)=\left(A\cos\left(\psi^i_c-a_cs^i_c\right),B\sin\left(\psi^i_c+b_cs^i_c\right)\right),$$
where $\psi^i_c$ is a constant . We can rewrite $\psi^i_c-a_cs^i_c=\psi^i_c+a_ds^i_d-a_cs^i_c-a_ds^i_d$ and $\psi^i_c+b_cs^i_c=\psi^i_c+b_cs^i_c-b_ds^i_d+b_ds^i_d$.
From \eqref{eqn_s_transform}, substituting $s^i_c=\frac{a_d+b_d}{a_c+b_c}s^i_d$ leads to
$$\psi^i_c-a_cs^i_c=\psi^i_d-a_ds^i_d \mbox{ and } \psi^i_c+b_cs^i_c=\psi^i_d+b_ds^i_d,$$where
\begin{align}
\psi^i_{d}=\psi^i_c+\frac{a_db_c-a_cb_d}{a_c+b_c} s^i_{d} \pmod{2\pi}. \label{eqn_psi_transform}
\end{align}
  
Thus $(x_i,y_i)=\left(A\cos\left(\psi^i_d-a_ds^i_d\right),B\sin\left(\psi^i_d+b_ds^i_d\right)\right)$ is unchanged under this transformation. Note that the number pairs $(a_c,b_c)$ and $(a_d,b_d)$ are  co-prime and both $a_c=a_d$ and $b_c=b_d$ cannot  hold simultaneously as ($N_c\neq N_d$). This implies that $\frac{a_c}{b_c}\neq\frac{a_d}{b_d}$   which means $\frac{a_db_c-a_cb_d}{a_c+b_c}\neq0$.
Thus we see that \eqref{eqn_s_transform} and \eqref{eqn_psi_transform} transform the parameters $s^i_c,\psi^i_c$ of the Lissajous curve with $(a_c,b_c,o_c)$ to the parameters $s^i_d,\psi^i_d$ of the Lissajous curve with $(a_d,b_d,o_d)$ without affecting the position coordinates of the agents. Since agents are at rest, from  \eqref{eqn_s_transform}, $\dot{s}^i_d=\dot{s}^i_c=0$ and $\dot{\psi}^i_d=\dot{\psi}^i_c=0$ for all $i\in\{1,...,N_c\}$.

\begin{remark}
The $\psi$ parameter separation between agents $i$ and $j$ along the ellipse is remains unchanged under the parameter transformation in \eqref{eqn_psi_transform} ,i.e.,
\begin{align}
\Delta^{ij}\psi(t)=\psi_c^j(t)-\psi_c^i(t)=\psi_d^j(t)-\psi_d^i(t).
\label{eqn_delta_ij_psi}
\end{align}
\end{remark}

 \subsubsection{ Symmetric transition of $\psi_d$}  
 \label{sec_sym_tr_psid}
For both agent addition and removal operations the formation agents move from the Lissajous curve for $N_c$ agents to the Lissajous curve for $N_d$ agents. Separate cooperative leader selection and transition assignment schemes are proposed for either operation in later sections that guarantee collision-free transition trajectories. These schemes assign a  destination  parameter value $\psi^i_D$ (on the Lissajous curve for $N_d$ agents) to each agent $i$, from the set 
\begin{align}
\Psi_D=\left\lbrace\frac{2\pi(p-1)}{N_d}+\frac{o_d\pi}{2}\bmod{2\pi}: p\in\{1,...,N_d\}\right\rbrace.
\label{eqn_set_psi_D}
\end{align}
These correspond to the agent formation positions on the Lissajous curve for $N_d$ agents from \eqref{eqn_alt_pos_rep}. Suppose the  formation agent $i$ is assigned destination parameter values at time $t=T^i_\Psi$. Each formation agent $i$ must travel the parametric displacement
to reach  the assigned $\psi^i_D$  value equal to  
\begin{align}\Delta^i_\psi=\psi^i_D-\psi^j_d(T_\Psi),
\label{eqn_delta_i_psi}
\end{align}
 For both addition and removal, the cooperative leader selection and transition assignment schemes also communicate necessary information to the formation agents that allows them to compute \begin{align}
 \Delta_{max}=\max_{i\in\{1,...,N_d\}} \vert \Delta^{i}_\psi\vert.
 \label{eqn_delta_max}
 \end{align}
 The agents then use the symmetric transition trajectory given by  \eqref{eqn_g_s_tr} to move to the positions corresponding to the assigned $\psi^i_D$  values. For this transition, $\dot{s}^i_d=0$ and as a result the agent speed depends only on $\vert \dot{\psi}^i_d \vert$. From \eqref{eqn_V_alt_param_a2b2}, $V^i<\sqrt{A^2+B^2}\vert\dot{\psi}^i_d\vert$. Therefore, $V^i<V_{max}$, if  $\vert\dot{\psi}^i_d\vert\leq\dot{\psi}_{max}$, where $\dot{\psi}_{max}=\frac{V_{max}}{\sqrt{A^2+B^2}}$.
Thus from \eqref{eqn_sym_gdot_max}, considering peak parametric speed as  $\dot{g}_{max}=\dot{\psi}_{max}$ for parametric interval $\vert g_f-g_0\vert =\Delta_{max}$, rearrangement of the equation gives $$T_p= \frac{15\Delta_{max}}{8V_{max}}\sqrt{A^2+B^2}.$$ Using $T_p$ as transition time period in \eqref{eqn_g_s_tr} for the agent $i_m=\argmax_{i\in\{1,...,N_d\}} \vert\Delta^{i}_\psi\vert$ guarantees that its maximum parametric speed is $\dot{\psi}_{max}=\frac{15\Delta_{max}}{8T_p}$ (hence limiting physical speed below $V_{max}$).
We also use $T_p$ as the transition time period for the remaining formation agents. Thus from \eqref{eqn_sym_gdot_max}, for all $i\neq i_m$, $\vert\dot{\psi}^i_{d_{max}}\vert=\vert\dot{\psi}^i_{d_{min}}\vert=\frac{15\Delta^i_{\psi}}{8T_p}\leq\frac{15\Delta_{max}}{8T_p}=\dot{\psi}_{max}$.
Hence  speeds of all the agents are bounded above by $V_{max}$.

From \eqref{eqn_sym_traj_spec}, the constants that characterise the symmetric trajectory in \eqref{eqn_g_s_tr} for agent $i$ are
\begin{align}
T_0=T^{i}_{\psi},\ T_p= \frac{15\Delta_{max}\sqrt{A^2+B^2}}{8
V_{max}},\ g_0=\psi^i_d(T_0),\ g_f=\psi^i_D. 
\label{eqn_sym_tr_boundary}
\end{align}

For both addition and removal the destination parameter values in $\Psi_D$ are assigned such that the symmetric transitions along the ellipse are in the same direction. Assuming negligible communication delays, the transition start times for formation agents $i,j$ ($i\neq j $) satisfy $T^i_\Psi\approx T^j_\Psi=T_0$ . As a consequence we have the following result for the relative parametric displacement $\Delta^{ij}\psi$:

\begin{prop}
\label{prop_sym_mono_sep}
For formation agents $i,j$ moving along the symmetric transition  trajectories for parameter $\psi_d$ characterised by \eqref{eqn_sym_tr_boundary} for the time window $t\in[T_0,\ T_f]$,   $\Delta^{ij}\psi(t)$ given by \eqref{eqn_delta_ij_psi} is monotone in nature and achieves its extremal values at $t=T_0,T_f$.
\end{prop}
\begin{proof}
 Since symmetric transition trajectory for reconfiguration is initiated at $T_0=T^{i}_{\psi}\approx T^{j}_{\psi}$ and ends at $T_f=T_0+T_p$,  from \eqref{eqn_g_s_dot_tr}, the relative parametric speed between agents $i$ and $j$, is given by
\begin{align}
\Delta^{ij}{\dot{\psi}}(t)=\dot{\psi}^j_d(t)-\dot{\psi}^i_d(t)= \frac{30(\Delta^j_{\psi}-\Delta^i_{\psi})}{T_p^5}\Delta t^2(T_p - \Delta t)^2,
\label{eqn_psi_rel_vel}
\end{align}
where $\Delta t=t-T_0$, $\Delta^j_\psi$ and $\Delta^i_\psi$ are given by \eqref{eqn_delta_i_psi}.
Thus from \eqref{eqn_psi_rel_vel}, for  $t\in (T_0,\ T_f)$,
$ \Delta^{ij}{\dot{\psi}}(t)> 0 \mbox{, if }  \Delta_\psi^j-\Delta_\psi^i> 0 \mbox{ and }\Delta^{ij}{\dot{\psi}}(t)< 0 \mbox{, if }  \Delta_\psi^j-\Delta_\psi^i< 0 $. This proves that
$\Delta^{ij}{\psi}(t)={\psi}^j_d(t)-{\psi}^i_d(t)$ is monotone in nature and its extremal values are attained at  $t=T_0,T_f$. 
\end{proof}
\subsubsection{\  Monotone acceleration of  $s_d$}\label{sec_mono_sd}
This step is common to all three reconfiguration operations. After completing the symmetric transition trajectory along the elliptical locus for agent addition and removal operations, and agent exchange for the replacement operation, all the $N_d$ agents have reached their destination $\psi^i_D$ values on the Lissajous curve selected for $N_d$ agents, and are at rest. Now the agents accelerate along this curve using the monotone transition trajectory for parameter  $s_d$ given by \eqref{eqn_g_tr} to resume performing the proposed surveillance strategy with $N_d$ agents with parametric speed $\dot{s}_d=\dot{s}_{nom}$ according to \eqref{eqn_sdot_choice} for the new Lissajous curve with $(a_d,b_d,o_d)$. This is initiated by an initiator agent $i_I$ (operation specific)  via the communication links. We assume that the transition is initiated by the leader $i_L$ at $T_0=T_a^i$  for agent $i$ (where $T^i_a\approx T^j_a$ for $i\neq j$ and $i,j\in\{1,...,N_d\}$). From \eqref{eqn_mono_traj_spec} the constants that specify this trajectory  in \eqref{eqn_g_tr} for  each agent  $i$ are
\begin{align}
T_0=T_a^i, \ g_0=s^i_d(T_0) , \  g_f=g_0+\frac{\pi}{8N_d},\  \dot{g}_0=0,\   
 \dot{g}_f=\dot{s}_{nom}.
\end{align}



 \subsubsection{ Symmetric transition to way-point} \label{sec_sym_wp} Here an agent moves from its initial position $P_0$ with coordinates $(x_0,y_0)$, to a final way-point position $P_f$ with coordinates $(x_f,y_f)$ by performing a symmetric transition trajectory (given by \eqref{eqn_g_s_tr}) for the displacement along the vector $\overrightarrow{P_0P_f}$ ( having length $ d_f=\sqrt{(x_f-x_0)^2+(y_f-y_0)^2}$ ). Assuming this transition is done over a time window $[T_0,\ T_f]$ with $T_f=T_0+T_p$ where  $T_p\geq\frac{15d_f}{8V_{max}}$, then from \eqref{eqn_sym_gdot_max} selecting the constants 
$$T_0=T^{i_a}_0,\ T_p=T^{i_a}_p,\ g_0=0 \mbox{ and }\ g_f=d_f$$
in \eqref{eqn_sym_traj_spec} for the trajectory given by  \eqref{eqn_g_s_tr}, ensures that the agent speed is bounded above by $V_{max}$.

We now discuss the  reconfiguration steps involved in each of the three operations separately. 

\subsection{ Agent removal} 
\label{sec_agent_removal}
The removed agent $i_r$ has two adjacent neighbours $i_p$ and $i_n$ on the formation ellipse  with parameter values $\psi^{i_n}_c=\psi^{i_r}_c+\frac{2\pi}{N_c} \bmod{2\pi}$ and $\psi^{i_p}_c=\psi^{i_r}_c-\frac{2\pi}{N_c} \bmod{2\pi}$ respectively. The algorithmic  sketch of the steps for the agent removal operation are as follows: 
\begin{algorithmic}[1]
\Statex {\bf Initial condition: } Proposed formations of $N_c$ agents  moving on the Lissajous curve  for surveillance at altitude $h_F$.
\State {\bf Removal initialisation: }Agent $i_r$ stops communication lowers altitude at  $t=T_R$ to $h_L<h_F$,  and returns to  base to land. Number of formation agents remaining $N_d=N_c-1$. 
\State  {\bf Monotone deceleration of parameter $s_c$:} initiated by agent $i_n$ for $N_d$ formation agents. 
\State{\bf Parameter transformation: }  from $(s_c,\psi_c)$ to $(s_d,\psi_d)$, done when  $\dot{s}^i_c=0$ for all formation agents $i$.
\State{\bf Leader selection:} Leader agent $i_L$ selected from $\{i_n,i_p\}$
\State{\bf Transition assignment: }Destination positions on the Lissajous curve for $N_d$ agents are assigned to formation agents by leader $i_L$
\State{\bf Symmetric transition trajectory of  parameter $\psi_d$: }Agents move along formation ellipse to reach assigned destination positions on the Lissajous curve for $N_d$ agents.
\State{\bf Monotone acceleration of parameter $s_d$:} initiated by agent $i_n$ after  the previous step, $N_d$ formation agents accelerate  along the Lissajous curve for $N_d$ agents to resume area surveillance. 
\end{algorithmic}
 The reconfiguration steps for the formation agents unique to the agent removal operation are as follows:\vspace{-0.1cm}
\subsubsection{  Removal initialisation}\label{sec_removal_init}
When agent $i_r$ is removed from the formation of $N_c$ agents at time $t=T_R$, its next agent $i_n$  alerts the remaining agents about the removal of agent $i_r$ via the communication links. We assume that the remaining $N_d=N_c-1$ formation agents (having indices $j\in\{1,...,N_c\}\setminus\{i_r\}$) are updated about the removal at time $t=T_R^j\approx T_R$. This is followed by the monotone deceleration of the parameter $s_c$ for which agent $i_n$ is the initiator agent $i_I$ (refer Section \ref{sec_mono_sd}).

\subsubsection{  Leader selection and transition assignment}\label{sec_leader_rem}
For the agent removal operation $N_d=N_c-1$, assuming agent $i_r$ is removed, the leader agent $i_L$ is selected from the agents $i_n$ and $i_p$ (neighbours of $i_r$ in the formation with parameter values  $\psi^{i_n}_c=\psi^i_c+\frac{2\pi}{N_c} \bmod{2\pi}$  and $\psi^{i_p}_c=\psi^i_c-\frac{2\pi}{N_c} \bmod{2\pi},$ respectively). If agent $i_n$ is selected as the leader (as shown in Fig \ref{fig_reconfig_rem}) then after the  parameter transformation, the  direction of transition  along the ellipse is selected in the  direction of the increasing $\psi_d$  parameter  (,i.e., $\dot{\psi}_d>0$) as shown in the Fig.\ \ref{fig_reconfig_rem}, as it is guaranteed to have an adjacent agent within its communication range of $r_{com}$ (at parameter $\psi_c=\psi^{i_n}_c+\frac{2\pi}{N_c} \bmod{2\pi}$ before the  parameter transformation).
Similarly, if agent $i_p$ is selected as the leader then the  direction of transition  along the ellipse is selected in the  direction of the decreasing $\psi_d$ parameter (i.e., $\dot{\psi}_d<0$).
 \begin{figure}[!h]
\centering
\includegraphics[width=0.8\linewidth]{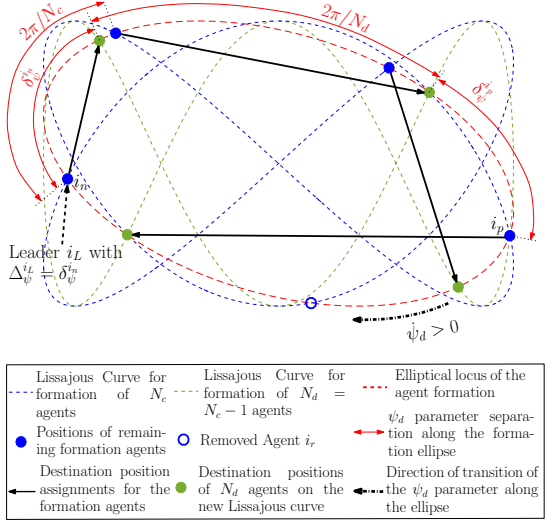}\\
\caption{Transition assignment example for the removal of an  agent from a $N_c=5$ agent formation for leader $i_L=i_n$.}
\label{fig_reconfig_rem}
\end{figure}
The destination parameter values on the Lissajous curve for $N_d$ agents lie in the set $\Psi_D$  given by \eqref{eqn_set_psi_D}. The element in $\Psi_D$ which is nearest to $i_L$ along the chosen direction of transition is selected as the destination parameter value for the leader $i_L$. For agent $i_n$ (shown in Fig.\ \ref{fig_reconfig_rem}), the nearest element of $\Psi_D$ encountered along the $\dot{\psi}_d>0$ direction is at $\psi^n_{cl}=\left\lceil\left(\psi^{i_n}_d-\frac{o_d\pi}{2}\right)\frac{N_d}{2\pi}\right\rceil\frac{2\pi}{N_d}+\frac{o_d\pi}{2}$, and parametric interval $\delta^{i_n}_\psi=\psi^n_{cl}-\psi^{i_n}_d$ as shown in Fig.\ \ref{fig_reconfig_rem}.
 Similarly for agent $i_p$, the nearest element of $\Psi_D$ encountered along the $\dot{\psi}_d<0$ direction is at $\psi^p_{cl}=\left\lfloor\left(\psi^{i_p}_d-\frac{o_d\pi}{2}\right)\frac{N_d}{2\pi}\right\rfloor\frac{2\pi}{N_d}+\frac{o_d\pi}{2} $, and  parametric interval $\delta^{i_p}_\psi=\psi^{i_p}_d-\psi^p_{cl}$ as shown in Fig.\ \ref{fig_reconfig_rem}. Both agents $i_n$ and $i_p$ compute the values of $\delta^{i_p}_\psi$ and $\delta^{i_n}_\psi$, and the leader agent $i_L$ is selected as follows:
 \begin{align}
i_L=\begin{cases} i_n, \mbox{ if }\delta^{i_n}_\psi\leq \delta^{i_p}_\psi,
\\
i_p, \mbox{ if }\delta^{i_n}_\psi> \delta^{i_p}_\psi.
\end{cases}
\label{eqn_lead_select_rem}
\end{align}
  The destination parameter $\psi^{i_L}_D$ and the corresponding parametric distance $\vert\Delta^i_\psi\vert$ (given by \eqref{eqn_delta_i_psi}) for the leader $i_L$   is selected as
\begin{align}
(\psi^{i_L}_D,\vert\Delta^{i_L}_\psi\vert) =\begin{cases} (\psi^n_{cl}, \delta^{i_n}_\psi),  \mbox{ if }i_L=i_n,  \\ (\psi^p_{cl}, \delta^{i_p}_\psi), \mbox{ if }i_L=i_p. \end{cases}
\label{eqn_dest_leader}
\end{align}
As a consequence of choices in  \eqref{eqn_lead_select_rem} and \eqref{eqn_dest_leader}, the selected leader has  smaller value of the parametric transition distance  $\vert\Delta^{i_L}_\psi\vert$ assigned for reconfiguration. The leader communicates  $\psi^{i_L}_d$ and $\psi^{i_L}_D$ values to the remaining formation agents using the communication links between adjacent agents.
  The subsequent values in $\Psi_D$ are assigned as destination parameter values to the subsequent agents, in the sequence in which they are encountered on the formation ellipse along the $\dot{\psi}_d<0$ direction for $i_L=i_p$ and $\dot{\psi}_d>0$ direction for $i_L=i_n$ (shown in Fig.\ \ref{fig_reconfig_rem}). Mathematically this can be written as
 \begin{align}
\psi^{i}_D=\begin{cases}\psi^{i_L}_D+\frac{2\pi n_i}{N_d}\bmod{2\pi},\mbox{ if }i_L=i_n, \\
\psi^{i_L}_D-\frac{2\pi n_i}{N_d}\bmod{2\pi},\mbox{ if }i_L=i_p,
\end{cases}
\label{eqn_dest_non_leader}
\end{align}
where $n_i$ is the count of agent $i$ relative to $i_L$ along the transition direction, and is computed as
\begin{align}
 n_i=\begin{cases}(\psi^i_d-\psi^{i_L}_d \bmod{2\pi} )\frac{N_c}{2\pi}, \mbox{ if } i_L=i_n,\\
 		(\psi^{i_L}_d-\psi^i_d \bmod{2\pi} )\frac{N_c}{2\pi},\mbox{ if } i_L=i_p.\end{cases}
 		\label{eqn_ni_rem}
 \end{align}
 Here $n_i\leq N_d-1$ is an integer because prior to any transition along the ellipse,  $\vert \psi^i-\psi^{i_L}\vert$ is an integer multiple of $\frac{2\pi}{N_c}$ for the agents $i\in\{1,...,N_c\}\setminus\{i_r,i_L\}$.

For the destination parameter assignment given by \eqref{eqn_dest_leader} and \eqref{eqn_dest_non_leader}, $\vert\Delta^{i}_\psi\vert$ (given by \eqref{eqn_delta_i_psi}) is the parametric transition distance that each agent must move along the formation ellipse to reach the assigned destination value $\psi^i_D$. Suppose the leader agent $i_L=i_n$. Then  as illustrated in Fig.\ \ref{fig_reconfig_rem}, the values of $\vert\Delta^{i}_\psi\vert$ for agents $i\in\{1,...,N_c\}\setminus\{i_r\}$ are given by
\begin{align}
\vert\Delta^{i}_\psi\vert=
	\vert\Delta^{i_L}_\psi\vert+n_i\left(\frac{2\pi}{N_d}- \frac{2\pi}{N_c}\right).
	\label{eqn_del_ni_psi_rem}
\end{align}
Equation \eqref{eqn_del_ni_psi_rem}  also holds for the case $i_L=i_p$. From \eqref{eqn_del_ni_psi_rem}, for the agent removal case, the longest parametric transition distance $\Delta_{max}$ in \eqref{eqn_delta_max} is obtained for $n_i=N_d-1$ and $N_d=N_c-1$, and is given by $$\Delta_{max}=\vert\Delta^{i_L}_\psi\vert+2\pi\frac{N_c-2}{N_c(N_c-1)}.$$
This is followed by the symmetric transition in the $\psi_d$ parameter for the $N_d$ formation agents as discussed in Section \ref{sec_sym_tr_psid}.

\begin{claim} \label{clm_no_col_rem}
The destination parameter assignment scheme given by \eqref{eqn_dest_non_leader} guarantees collision-free symmetric transition trajectories in parameter $\psi_d$ for the $N_d$ formation agents.
\end{claim}
\begin{proof}
 Prior to the agent removal,  $\vert\Delta^{ij}\psi(t)\vert=\vert\psi_c^j(t)-\psi_c^i(t)\vert=\frac{2\pi}{N_c}$ for adjacent agents $i$ and $j$. We assume that agent $j$ succeeds agent $i$ along the selected direction of transition, i.e., $i,j$ satisfy $n_j=n_i+1$ in \eqref{eqn_ni_rem}. 
Thus in the time interval $[T_0,\ T_f]$ for the symmetric transition of formation agents, with $T_0=T^i_\Psi \approx T^j_\Psi$ and $T_f=T_0+T_p$, the initial and final parametric displacement (in parameter $\psi_d$) from agent $i$ to the adjacent agent on the formation ellipse along the direction of transition is given by:
\begin{align}
 \Delta^{ij}\psi(T_0)=\begin{cases}\frac{4\pi}{N_c} \mbox{ if }i= i_p\\
 \frac{2\pi}{N_c} \mbox{ if }i\neq i_p 
  \end{cases}\hspace{-0.4cm}, \Delta^{ij}\psi(T_f)= \frac{2\pi}{N_d} \mbox{ for } i_L=i_n,\label{eqn_deltaij_psi_ilin_rem}\\ 
   \Delta^{ij}\psi(T_0)=\begin{cases}\frac{-4\pi}{N_c} \mbox{ if }i= i_n \\
 \frac{-2\pi}{N_c} \mbox{ if }i\neq i_n
  \end{cases}\hspace{-0.4cm},\Delta^{ij}\psi(T_f)= \frac{-2\pi}{N_d} \mbox{ for } i_L=i_p. \label{eqn_deltaij_psi_ilip_rem}
 \end{align}

From Proposition \ref{prop_sym_mono_sep}, we know that $\Delta^{ij}\psi(t)$ is monotone for  $t\in [T_0,\ T_f]$ and attains its maximum and minimum values at $T_0$ or $T_f$. Thus from \eqref{eqn_deltaij_psi_ilin_rem} and \eqref{eqn_deltaij_psi_ilip_rem},  the  minimum value of $\vert\Delta^{ij}\psi(t)\vert=\Delta^{ij}_{min}=\frac{2\pi}{N_c}$ for the agent removal operation.
From Proposition \ref{prop_ellipse_avoid} and   \eqref{eqn_rdm_reconfig}, 
$\delta_s=\frac{\pi}{2N_c}\frac{\sqrt{A^2+B^2}}{\sqrt{A^2 a_j^2+B^2 b_j^2}}\frac{\sin\left(\frac{\pi}{N_{max}}\right)}{\left\vert\sin\left(\frac{\pi }{N_c}\right)\right\vert}.$

Since $N_c \leq N_{max}$ and $(a_j,b_j)$ are co-prime positive integers, $\delta_s<\frac{\pi}{2N_c}$. Moreover, as the formation agents stop on a feasible ellipse with $s_c\neq S_{avoid}$ (which is ensured  as discussed in Section \ref{sec_mono_sc}), the symmetric transition trajectories in the parameter $\psi_d$ for the proposed transition assignment scheme are collision-free. 
\end{proof}\vspace{-0.7cm}

\subsection{ Agent addition }
\label{sec_agent_addition}
 For the agent addition operation, $i_p$ and $i_n$ are the formation agents parametrically preceding and succeeding the assigned formation position for added agent $i_a$ respectively. The algorithmic  sketch of the steps for adding agent $i_a$ to a formation of $N_c$ agents to get a formation of $N_d=N_c+1$ agents  are as follows: 

\begin{algorithmic}[1]

\Statex {\bf Initial condition: } Proposed formations of $N_c$ agents  moving on the Lissajous curve  for surveillance at altitude $h_F$.
\State {\bf Addition Initialisation: }Agent $i_a$ waiting at height $h_L$, alerts the closest formation agent $i_c$ to initiate the agent addition operation when in communication range.
\State  {\bf Monotone deceleration of parameter $s_c$:} initiated by agent $i_c$ at $t=T_R$ for all formation agents. 
\State {\bf Parameter update for added agent:}  The formation agents  cooperatively calculate the $s^{i_a}_d$, $\psi^{i_a}_d$ parameters for the formation position of agent $i_a$, and communicate the same to $i_a$.
$i_a$ moves to its assigned formation position at height $h_L$

\State{\bf Parameter transformation: }  From $(s_c,\psi_c)$ to $(s_d,\psi_d)$, done for formation agents $i$ at height $h_F$ when  $\dot{s}^i_c=0$.
\State{\bf Leader selection:} Leader agent $i_L$ is selected from $\{i_n,i_p\}$
\State{\bf Transition assignment: }Destination positions on the Lissajous curve for $N_d$ agents are assigned to formation agents by $i_L$.
\State{\bf Symmetric transition trajectory of  parameter $\psi_d$: }Agents move along formation ellipse to reach assigned destination positions on new Lissajous curve.
\State{\bf Agent $i_a$ enters formation:} During the symmetric transition trajectory, $i_a$ rises to height $h_F$ when the ascent is collision-free.
\State{\bf Monotone acceleration of parameter $s_d$:} initiated by agent $i_a$ after  the previous two steps. The $N_d$ formation agents accelerate  along the Lissajous curve for $N_d$ agents to resume area surveillance. 

\end{algorithmic}
 We now discuss the reconfiguration steps  specific to the addition operation for the  added agent $i_a$ and the formation agents.
\subsubsection{  Addition initialisation}\label{sec_add_init} The agent to be added to the formation is launched from the home base and it hovers at altitude $h_L <h_F$. When formation agent $i_c$ is in communication range, it is alerted by $i_a$ to initiate the agent addition operation. Agent $i_c$ then initiates the monotone deceleration of the parameter $s_c$ for the $N_c$ formation agents.

\subsubsection{  Parameter update for added agent}\label{sec_param_update_ia}
The formation agents calculate the position of entry on the elliptical locus for the added new agent and communicate the corresponding $\psi_d$ and $s_d$ values to the added agent. The calculation of these values is done considering Lemma \ref{lem_ellipse_partitions} (proof given in the Appendix).

\begin{lemma}
Suppose $P_1,P_2,...,P_N$  are $N$ parametrically equi-spaced points on a closed curve $C $ with the convention $P_{N+1}=P_1$, and  $Q_1,Q_2,...,Q_{N+1}$  are $N+1$ parametrically equi-spaced points on the same curve $C$ with  convention $Q_{N+2}=Q_1$. Then there is exactly one pair of adjacent $Q$ points contained in the interval $[P_{i},\  P_{i+1})$ of adjacent $P$ points for some $i \in \{1,...,N\}$.
\label{lem_ellipse_partitions}
\end{lemma}

From \eqref{eqn_alt_pos_rep}, we know that for the proposed formation, the agent positions on the Lissajous curve partition the elliptical locus in parametrically equal parts in terms of parameter $\psi_c$. Thus from Lemma \ref{lem_ellipse_partitions}, at any point in time exactly two adjacent agent positions on the Lissajous curve for $N_d=N_c+1$  agents (shown as green spots in Fig.\ \ref{fig_reconfig_add}) must lie between two adjacent agent positions on the Lissajous curve for $N_c$ agents (shown as blue spots in Fig.\ \ref{fig_reconfig_add}) along the elliptical locus.
\begin{figure}[!h]
\centering
\includegraphics[width=0.8\linewidth]{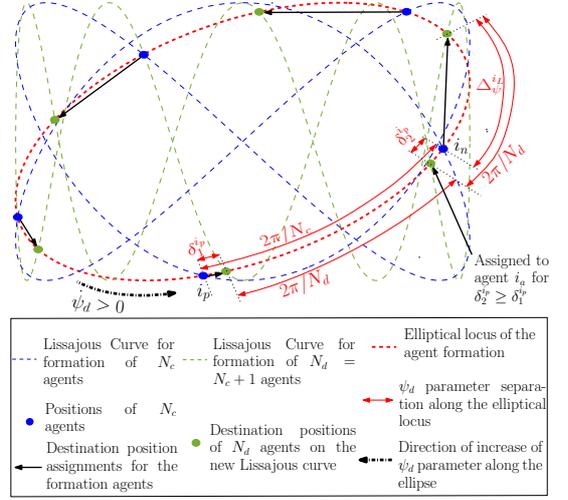}\\
\caption{Transition assignment example for the addition of an  agent from a $N_c=5$ agent formation for leader $i_L=i_n$.}
\label{fig_reconfig_add}
\end{figure}
Hence after the calculation of the stopping parameter value $s^i_f$ as discussed in Section \ref{sec_mono_sc},  each of the formation agents computes the transformed values $s^i_{d_f}$ and $\psi^i_d$ of the stopping parameter value $s^i_{f}=s^{i_c}_{f}$ and $\psi^i_c$ using \eqref{eqn_s_transform} and \eqref{eqn_psi_transform} respectively. Defining $\Psi_D$ as in \eqref{eqn_set_psi_D}, each formation agent then computes the following terms:
\begin{align}
\delta^i_1&=\left\lceil\left(\psi^{i}_d-\frac{o_d\pi}{2}\right)\frac{N_d}{2\pi}\right\rceil\frac{2\pi}{N_d}+\frac{o_d\pi}{2} -\psi^{i}_d,\label{eqn_delta_1}\\
\delta^i_2&=\frac{2\pi}{N_c}-\left(\delta^i_1+\frac{2\pi}{N_d} \right),\label{eqn_delta_2}
\end{align}
where $\delta^i_1$ is the  parameter separation between  $\psi^i_d$ and the closest destination position  on the Lissajous curve for $N_d$ agents (from the set $\Psi_D$) along the $\dot{\psi}_d>0$ direction.
As a consequence of Lemma \ref{lem_ellipse_partitions},
 the value of $\delta^i_2>0$  for exactly one of the formation agents (as shown in Fig.\ \ref{fig_reconfig_add}). We call this agent as agent $i_p$. Agent $i_p$ then selects the value of the destination parameter $\psi^{i_a}_D$ (from the set $\Psi_D$) for the entry position of the added agent $i_a$ as follows:
 \begin{align}
 \psi^{i_a}_D=\begin{cases} \psi^{i_p}_d+\delta^{i_p}_1 \mbox{, if } \delta^{i_p}_1>\delta^{i_p}_2\\
\psi^{i_p}_d+\frac{2\pi}{N_c}-\delta^{i_p}_2, \mbox{ if }\delta^{i_p}_1\leq\delta^{i_p}_2.
 \end{cases}
 \label{eqn_param_ia}
 \end{align}

It will be shown later in the transition assignment step that this choice ensures a shorter transition along the elliptical locus for the formation agents to reconfigure to the Lissajous curve for $N_d$ agents. Assuming that the reconfiguration is initiated at time $T_R$, the values of $N_c$,  $ \psi^{i_a}_D$, $s^{i_c}_c(T_R)$  and  $s^{i_a}_{cf}=s^{i_c}_{c_f}$  are communicated to agent $i_a$ by the formation agents. Agent $i_a$ transforms $s^{i_a}_{cf}$ to $s^{i_a}_{df}$ using \eqref{eqn_s_transform}, and also calculates coordinates of the entry point into the formation using \eqref{eqn_alt_pos_rep} as
\begin{align}
(x^{i_a}_E,y^{i_a}_E)=\left(A\cos( \psi^{i_a}_D- a_ds^{i_p}_{d_f}),\ B\sin( \psi^{i_a}_D+ b_ds^{i_p}_{d_f}) \right).
\label{eqn_entry_add}
\end{align}
It then computes the time period for monotone deceleration of $s^{i_c}_c$ for agent $i_c$ using \eqref{eqn_Tf_gtr} as $T^{i_c}_p=\frac{2(s^{i_c}_{cf}-s^{i_c}_c(T^i_R))}{\dot{s}_{nom}}$.
Agent $i_a$ then moves to  position given by \eqref{eqn_entry_add} using the symmetric transition for a way-point (discussed in Section \ref{sec_sym_wp}) with transition period $T_p=\max\left(T^{i_c}_p,\frac{15d_f}{8V_{max}}\right)$.
This ensures that the added agent $i_a$ does not reach its formation position before the formation agents decelerate to a halt with $\dot{s}_c=0$.

\subsubsection{  Leader selection and transition assignment} \label{sec_leader_add} After the parameter transformation step is completed for the formation agents at height $h_F$, and the agent $i_a$ has reached formation position given by \eqref{eqn_entry_add} at height $h_L<h_F$, the formation agents must be assigned a destination $\psi^j_D$ parameter values on the Lissajous curve for $N_d=N_c+1$ agents, from the set $\Psi_D$ given by \eqref{eqn_set_psi_D}, so that they can transition  along the formation ellipse t these locations for reconfiguration. This can be achieved by motion along the $\dot{\psi}_d>0 $ or the $\dot{\psi}_d<0 $ direction. 

The parameter value $\psi^i_d$ of the formation agent $i$ on  the Lissajous curve for $N_c$ agents ( corresponding to the blue spots in Fig.\ \ref{fig_reconfig_add}) and elements of set $\Psi_D$, corresponding to the location of the formation agents on the Lissajous curve for $N_d=N_c+1$ agents (green spots in Fig.\ \ref{fig_reconfig_add}),  are equi-spaced on the formation ellipse, with a parametric separation of $\frac{2\pi}{N_c}$ and $\frac{2\pi}{N_d}$ respectively. In Section \ref{sec_param_update_ia} from Lemma \ref{lem_ellipse_partitions}, agent $i_p$ was identified as the only formation agent  having exactly one pair of destination values from $\Psi_D$ in the interval $\left[\psi^{i_p}_d,\ \psi^{i_p}_d+\frac{2\pi}{N_c}\right)$ (i.e., with $\delta^{i_p}_2>0$ in \eqref{eqn_delta_2}). We call the  adjacent agent parametrically succeeding  $i_p$ as $i_n$, having parameter value $\psi^{i_n}_d=\psi^{i_p}_d+\frac{2\pi}{N_c} \bmod{2\pi}$. From \eqref{eqn_param_ia}, the formation position of agent $i_a$ lies between $i_p$ and $i_n$ on the formation ellipse. The leader agent to initialise the direction of transition assignment is chosen as
\begin{align}
i_L=\begin{cases} i_p, \mbox{ if }\delta^{i_p}_1>\delta^{i_p}_2\\
i_n, \mbox{ if }\delta^{i_p}_1\leq \delta^{i_p}_2
\end{cases}.
\label{eqn_leader_add}
\end{align} 
Thus  the destination values from $\Psi_D$ are assigned to the formation agents as the closest  value along $\dot{\psi}_d>0$ for $i_L=i_n$, and along $\dot{\psi}_d<0$ for $i_L=i_p$. This is mathematically written as
\begin{align}
\psi^i_D=\begin{cases}\left\lfloor\left(\psi^{i}_d-\frac{o_d\pi}{2}\right)\frac{N_d}{2\pi}\right\rfloor\frac{2\pi}{N_d}+\frac{o_d\pi}{2} \mbox{ if }i_L=i_p,\\
\left\lceil\left(\psi^{i}_d-\frac{o_d\pi}{2}\right)\frac{N_d}{2\pi}\right\rceil\frac{2\pi}{N_d}+\frac{o_d\pi}{2}\mbox{ if } i_L=i_n.
\end{cases}
\label{eqn_param_assign_add}
\end{align} 
As a consequence of Lemma \ref{lem_ellipse_partitions}, outside the segment $\left[\psi^{i_p}_d,\ \psi^{i_n}_d\right)$, the positions corresponding to values of $\psi^i_d$ and elements of $\Psi_D$ alternate along the elliptical locus as shown in Fig.\ \ref{fig_reconfig_add}. Thus this assignment yields parametrically non-overlapping transition intervals as shown in  Fig.\ \ref{fig_reconfig_add}. The destination parameter values from $\Psi_D$ are assigned according to \eqref{eqn_dest_non_leader} and \eqref{eqn_ni_rem} (as in  the removal case).

Also,  the magnitude of the parametric transition distance $\vert\Delta^i_\psi\vert$ (given by \eqref{eqn_delta_i_psi}) for agent $i$ can be expressed in terms of the parametric transition distance of the leader $\vert\Delta^{i_L}_\psi\vert$ using \eqref{eqn_del_ni_psi_rem}, where $n_i$ is calculated according to \eqref{eqn_ni_rem}.
%
If $i_L=i_n $ as shown in Fig.\ \ref{fig_reconfig_add}, then $\vert\Delta^{i_L}_\psi\vert=\vert\Delta^{i_n}_\psi\vert=\frac{2\pi}{N_d}-\delta^{i_p}_2$. Similarly if $i_L=i_p$, then $\vert\Delta^{i_L}_\psi\vert=\vert\Delta^{i_p}_\psi\vert=\frac{2\pi}{N_d}-\delta^{i_p}_1$. Thus the leader selection in \eqref{eqn_leader_add} ensures that the transition direction corresponding to the shorter parametric transition distance is selected. Since $N_d=N_c+1$ from \eqref{eqn_del_ni_psi_rem},  the longest transition interval for any formation agent for the agent addition case is given by
$$\Delta_{max}=\Delta^{i_L}_\psi.$$


This is followed by the symmetric transition in the $\psi_d$ parameter for the $N_c$ formation agents at formation height $h_F$ (as discussed in Section \ref{sec_sym_tr_psid}). 

\begin{claim} \label{clm_no_col_add}
The destination parameter assignment scheme given by \eqref{eqn_dest_non_leader} guarantees collision-free symmetric transition trajectories in parameter $\psi_d$ for the $N_c$ formation agents at formation height $h_F$.
\end{claim}
\begin{proof}
We discuss collision-free transition between  $i_a$ (waiting at formation height $h_L$) and leader $i_L$ separately in the subsection \ref{sec_ia_entry}. The remaining $N_c$ agents are at formation altitude $h_F$ and equi-parametrically spaced along the elliptical locus prior to any transitions in  parameter $\psi_d$ at time $t=T_0$. We assume agent $j$ succeeds agent $i$ along the selected direction of transition, i.e., $i,j$ satisfy $n_j=n_i+1$ in \eqref{eqn_ni_rem}. Thus the  initial value of  $\Delta^{ij}\psi$ (given in \eqref{eqn_delta_ij_psi}) for adjacent formation agents $i$ and $j$ prior to reconfiguration is
\begin{align}
 \Delta^{ij}\psi(T_0)=\begin{cases}\frac{2\pi}{N_c} \mbox{ if }i_L= i_n,\\
 \frac{-2\pi}{N_c} \mbox{ if }i_L= i_p .
  \end{cases}
  \label{eqn_deltaij_psi_T0_add}
 \end{align}
 After the symmetric transition is over, $i_a$ is at its formation position at height $h_F$ and we have $N_d$ agents at their assigned parameter values given by \eqref{eqn_param_assign_add}. As a consequence of Lemma \ref{lem_ellipse_partitions}, outside the segment $\left[\psi^{i_p}_d,\ \psi^{i_n}_d\right)$, the positions corresponding to values of $\psi^i_d$ and elements of $\Psi_D$ assigned by \eqref{eqn_param_assign_add} alternate along the elliptical locus as shown in Fig.\ \ref{fig_reconfig_add}. Thus after the completion of the symmetric transition, 
\begin{align}
\Delta^{ij}\psi(T_f)=\psi^j_D-\psi^i_D=\begin{cases}\frac{2\pi}{N_d} \mbox{ if }i_L= i_n,\\
 \frac{-2\pi}{N_d} \mbox{ if }i_L= i_p. 
  \end{cases}
  \label{eqn_deltaij_psi_Tf_add}
\end{align}
 As a consequence of Proposition \ref{prop_sym_mono_sep}, from \eqref{eqn_deltaij_psi_T0_add} and \eqref{eqn_deltaij_psi_Tf_add}, the  minimum value of $\vert\Delta^{ij}\psi(t)\vert=\Delta^{ij}_{min}=\frac{2\pi}{N_d}$ for the agent addition operation.
From Proposition \ref{prop_ellipse_avoid} and   \eqref{eqn_rdm_reconfig}, 
$\delta_s=\frac{\pi}{2N_c}\frac{\sqrt{A^2+B^2}}{\sqrt{A^2 a_j^2+B^2 b_j^2}}\frac{\sin\left(\frac{\pi}{N_{max}}\right)}{\left\vert\sin\left(\frac{\pi }{N_d}\right)\right\vert}$.

Since $N_d \leq N_{max}$ and $(a_j,b_j)$ are co-prime positive integers, $\delta_s<\frac{\pi}{2N_c}$. Moreover, as  the formation agents stop on a feasible ellipse with $s_c\neq S_{avoid}$ (which is ensured as discussed in  Section \ref{sec_mono_sc}),  the symmetric transition trajectories in the parameter $\psi_d$ for the proposed transition assignment scheme are collision-free.
\end{proof}
\subsubsection{ Entry of agent $i_a$ in the formation}\label{sec_ia_entry}
Recall that agent $i_a $ is  still waiting at its formation position at $h_L<h_F$. 
As a consequence of leader selection in \eqref{eqn_leader_add} and parameter assignment in  \eqref{eqn_param_ia}, agent $i_L$ is the nearest formation agent to agent $i_a$ in terms of the parametric separation in  $\psi_d$ before the commencement of the transition along the formation ellipse.
During the symmetric transition in parameter $\psi_d$, agent $i_a$ rises to formation altitude $h_F$ when its Euclidean distance from agent $i_L$ is greater  than $2r_{dm}$. This guarantees the collision-free rise of agent $i_a$ to result in a formation of $N_d=N_c+1$ agents at height $h_F$ at the end of the symmetric transition trajectory discussed in Section \ref{sec_sym_tr_psid}.

\subsection{ Agent replacement }\label{sec_agent_replacement}
The agent replacement operation replaces formation agent $i_r$ with a new agent $i_R$. The algorithmic  sketch of the steps for replacing an agent in the formation of $N_c$ agents  are as follows: 
 
\begin{algorithmic}[1]
\Statex {\bf Initial condition: } Proposed formations of $N_c$ agents  moving on the Lissajous curve  for surveillance.
\State {\bf Replacement Initialisation: }Agent $i_R$ initialised with id of agent $i_r$, and waiting at height $h_L<h_F$, alerts the closest formation agent $i_c$ to initiate the agent replacement operation. 
\State  {\bf Monotone deceleration of parameter $s_c$:} initiated by agent $i_c$  for all formation agents. 

\State {\bf Parameter update for agent $i_R$:}  Formation agent $i_r$ (at height $h_F$) communicates  the $s^{i_r}_c$, $\psi^{i_r}_c$ parameters to agent $i_R$. $i_R$ moves to the position of $i_r$ at height $h_L<h_F$. 

\State{\bf Position exchange of $i_r$ and $i_R$: } This is done after the previous two steps are complete. Agent $i_r$ returns to base and lands.
\State{\bf Monotone acceleration of parameter $s_d=s_c$:} initiated by agent $i_R$ after  the previous step, $N_c$ formation agents accelerate  along the Lissajous curve for $N_c$ agents to resume area surveillance. 

\end{algorithmic}
 We discuss the reconfiguration steps  unique to the replacement of the formation agent $i_r$ by agent $i_R$ below.

\subsubsection{ Replacement initialisation}\label{sec_rep_init} The agent $i_R$ is initialised with the id $i_r$ of the formation agent that it is meant to replace. Agent $i_R$ takes off from the base location and waits at height $h_L<h_F$ for the formation agents (at height $h_F$) to approach. When the closest formation agent $i_c$ is within communication range, agent $i_R$ alerts agent $i_c$ to initiate the agent replacement operation. Also,  agent $i_R$ communicates the id  of the formation agent $i_r$ to  the formation agents via agent $i_c$.
Agent $i_c$ then initiates the monotone deceleration of the parameter $s_c$ (discussed in Section \ref{sec_mono_sc} ) for the $N_c$ formation agents.


\subsubsection{ Parameter update of agent $i_R$}\label{sec_iR_update} With the initiation of the monotone deceleration trajectory for $s_c$, 
the agent $i_r$ sends its parameter  $\psi^{i_r}_c$, $s^{i_r}_c(T_R)$ and   parameter $s^{i_r}_{cf}$  corresponding to the stopping formation ellipse, to  agent $i_R$ using the communication links via agent $i_c$.
Using \eqref{eqn_alt_pos_rep}, the agent $i_R$ at height $h_L$ calculates coordinates directly below agent $i_r$ at height $h_F$ as
\begin{align}
(x^{i_R}_E,y^{i_R}_E)=\left(A\cos( \psi^{i_r}_c- a_cs^{i_r}_f),\ B\sin( \psi^{i_r}_c+ b_cs^{i_r}_f) \right).
\label{eqn_entry_rep}
\end{align}

Similar to the addition case (in Section \ref{sec_param_update_ia}) the agent $i_R$ computes time period for monotone deceleration of $s^{i_r}_c$  using \eqref{eqn_Tf_gtr} as $T^{i_r}_p=\frac{2(s^{i_r}_{cf}-s^{i_r}_c(T^i_R))}{\dot{s}_{nom}}$, and moves to the 
  position given by \eqref{eqn_entry_rep} using the symmetric transition for a way-point (discussed in Section \ref{sec_sym_wp}) with transition period $T_p=\max\left(T^{i_r}_p,\frac{15d_f}{8V_{max}}\right)$.
%
%
%

\subsubsection{ Position exchange of $i_r$ and $i_R$ } \label{sec_ir_iR_swap}After the formation agents decelerate to rest at height $h_F$, and agent $i_R$ reaches  the  position coordinates \eqref{eqn_entry_rep} at height $h_L<h_F$ (say at time $t=T^s_0$), then the agents $i_R$ and $i_r$ are both at rest at the same position coordinates,  and are separated in altitude by distance $h_F-h_L$. In Section \ref{sec_ellipse_formation}, we have seen that for a fixed value of $s=s_o$, $\mathcal{E}(\psi)=[A\cos(\psi-as_o)\ B\sin(\psi+bs_o)]^T$ is the parametric equation of an ellipse  with parameter $\psi$. Thus the tangent vector is given by $\mathcal{T}(\psi)= \left[-A\sin(\psi-a_cs_o)\  B\cos(\psi+b_cs_o))\right]^T.$ Consider the vector $\mathcal{N}(\psi)=\pm[B\cos(\psi+b_cs_o)\ A\sin(\psi-a_cs_o)]^T$. Then the inner product $\langle \mathcal{N}(\psi),\mathcal{T}(\psi)\rangle=0$ for all $\psi$, which implies that it gives the direction of the local normal to the ellipse. If $\mathcal{N}(\psi)$ is  an outward normal to the elliptical locus at the point $\mathcal{E}(\psi)$, then $\langle \mathcal{N}(\psi), \mathcal{E}(\psi)\rangle=AB\cos(N_c s_o) \geq 0$, because  the elliptical locus is always centered at the origin, and is a convex curve. Thus to ensure the selection of the outward normal, $\mathcal{N}(\psi)$ is chosen as
\begin{align}
\mathcal{N}(\psi)= \sign \left(\cos(N_c s_o)\right)\left[\begin{matrix}B\cos(\psi+bs_o)\\ A\sin(\psi-as_o)\end{matrix}\right]
\end{align}
where $\sign(x)=\begin{cases} 1, \mbox{ if }x\geq 0\\  -1, \mbox{ if }x< 0\end{cases} $.

The agent $i_r$ computes a way-point on unit outward normal direction $\hat{N}(\psi)=\frac{N(\psi)}{\parallel N(\psi)\parallel}$  at distance $d^{i_r}_f=3r_{dm}$, and moves to this point using the symmetric trajectory to a waypoint (discussed in Section \ref{sec_sym_wp})
with  $T_p=\frac{15d_f}{4V_{max}}$.
Upon completion of this motion, the agent $i_r$ alerts agent $i_R$, which rises to the formation height $h_F$ and agent $i_r$ simultaneously reduces its altitude to $h_L$. The formation containing agent $i_R$, then performs monotone acceleration of parameter $s_d=s_c$ along the Lissajous curve for $N_c$ agents (discussed in Section \ref{sec_mono_sd}) and agent $i_r$ returns to base to land.

%
\begin{remark}
In both reconfiguration operations of removal and replacement, the removed or replaced agent $i_r$ is at a height $h_L<h_F$ at the end of the reconfiguration and is  made to return to the base and land, using the  symmetric trajectory to a way-point (discussed Section \ref{sec_sym_wp}) with $T_p=\frac{15d_f}{8V_{max}}$.
\end{remark}


\section{Simulation and experimental validation \vspace{-0.2cm}}\label{sec_sim_expt}
The proposed surveillance strategy of using an elliptical formation of multiple agents on a Lissajous curve discussed in Section \ref{sec_surveillance} was validated through simulation and experiments with differential drive robots in  prior work in \cite{lisiros}. The video of these experiments and simulation can be found at the web-link:\\ \small{https://youtu.be/rhygE32UDO8} 

The reconfiguration scheme for the formation discussed in Section \ref{sec_recon} is  validated here by simulation in MATLAB\reg for parametric agents (agents whose positions are defined by \eqref{eqn_alt_pos_rep}) having finite  non-zero sizes. In order to achieve a decentralized   implementation of the reconfiguration strategy on actual quadrotors, we first develop and test the on-board software for path planning and inter-agent communication for the quadrotors as agents, using a Software-In-The-Loop (SITL) simulator in a ROS-Gazebo environment. The same software is then used on the actual quadrotors (developed in-house) to experimentally validate the reconfiguration strategy   in a motion capture environment.
\begin{table}[h!]
\centering
\caption{Inputs to Algorithm \ref{algo_initializer}. (Distances in meters)}
\label{tab_algo_ip}
\resizebox{7.5cm}{!}{\tiny
\begin{tabular}{|l|l|l|l|l|l|l|l|}
\hline
                                                              & $L$ & $H$ & $r_s$ & $r_{com}$ & $V_{max}$ & $N_{extra}$ & $\eta$ \\ \hline
\begin{tabular}[c]{@{}l@{}}MATLAB\\ Simulation 1\end{tabular} & 10      & 7       & 4.7         & 9.5             & 0.5                  & 2           & 1.05   \\ \hline
\begin{tabular}[c]{@{}l@{}}MATLAB\\ Simulation 2\end{tabular} & 10     & 7       &     1.5        &        3.2         &                    1 &     2        &  1.05      \\ \hline
\begin{tabular}[c]{@{}l@{}}SITL \\ Simulation\end{tabular}    &     25      &     16      &       7      &     11     &       0.3              &        1     &       1.05 \\ \hline
Experiment                                                    &       5    &  5         &   2.7          &      5.5           &    0.2  &  1             &    1.05    \\ \hline
\end{tabular}
}
\end{table}

\begin{table}[!h]
\centering
\caption{Output of Algorithm \ref{algo_initializer}. (Distances in meters)}
\label{tab_algo_op}
\resizebox{7.5cm}{!}{
\begin{tabular}{|l|l|l|l|l|l|l|l|l|l|l|}
\hline
                                                              & $A$ & $B$ & $a$ & $b$ & $o$             & $N$ & $N_{min}$ & $N_{max}$ & $\dot{s}_{nom}$ & $r_{dm}$ \\ \hline
\begin{tabular}[c]{@{}l@{}}MATLAB\\ Simulation 1\end{tabular} & 5   & 3.5 & 2   & 3   & $\frac{\pi}{2}$ & 5   & 4         & 6         & 0.0345          & 0.481    \\ \hline
\begin{tabular}[c]{@{}l@{}}MATLAB\\ Simulation 2\end{tabular} & 5   & 3.5 & 4   & 11  & $\frac{\pi}{2}$ & 15  & 14        & 16        & 0.023           & 0.074    \\ \hline
\begin{tabular}[c]{@{}l@{}}SITL\\ Simulation\end{tabular}     & 12.5  & 8   & 3   & 7   & 0               & 10  & 9         & 10        & 0.0045          & 0.459    \\ \hline
Experiment                                                    & 2.5 & 2.5 & 3   & 2   & 0               & 5   & 4         & 5         & 0.0222          & 0.407    \\ \hline
\end{tabular}
}
\end{table}
\subsection{MATLAB\reg simulations}

To validate the theory  of the proposed surveillance (discussed in Section \ref{sec_surveillance}) and reconfiguration strategy (discussed in Section \ref{sec_recon}), we present two MATLAB\reg simulations: The first to illustrate the nature of the agent trajectories and the second to demonstrate scalability. The inputs to  Algorithm \ref{algo_initializer} for both simulations are given in Table \ref{tab_algo_ip} and the corresponding outputs are given in Table \ref{tab_algo_op}. The addition of agent 6 to a 5 agent formation at time $t=12\ sec$ is  considered for discussion here, and the simulated trajectories shown in the Fig.\ \ref{fig_sim_traj} include the speeds of the agents, the parameter rates $\dot{\psi},\ \dot{s}$ and parameter values $\psi,\ s$ of all agents. As illustrated by the Fig.\ \ref{fig_sim_traj}, the monotone parametric trajectory \eqref{eqn_g_tr} smoothly accelerates and decelerates the agents along the Lissajous curve and the symmetric trajectory \eqref{eqn_g_s_tr} smoothly transitions the agents from one Lissajous curve to the other. Furthermore, as discussed in the theory, the speeds of the agents are always maintained below $V_{max}=50\ cm/s $ in the X-Y plane. (For altitude changes we use step change commands when X-Y plane velocities are zero). The parameter trajectories are smooth except for the step just before the $\Delta T_3$ interval. This  jump in the parameter value corresponds to the parameter transformation step where the agent parameters $\psi$ and $s$ for the current Lissajous curve are expressed in terms of the destination Lissajous curve $(\psi_d,s_d)$ while conserving position coordinates (as discussed in Section \ref{sec_param_tx}). The  $s$ parameter trajectories of agents 1-5 overlap as they all lie on the same formation ellipse, both before and after the reconfiguration. The $s$ parameter value of the agent 6 is initialised on the destination Lissajous curve as it reaches its assigned formation position directly after the $\Delta T_1$ transition. The parameter values $\psi$ prior to the addition at $t=12\ sec$ and after reconfiguration $t>60\ sec$ are equispaced, indicating equi-parametric formation along the elliptical locus on both the Lissajous curves, before and after reconfiguration.
\begin{figure}[h!]
\centering
\includegraphics[width=0.9\linewidth]{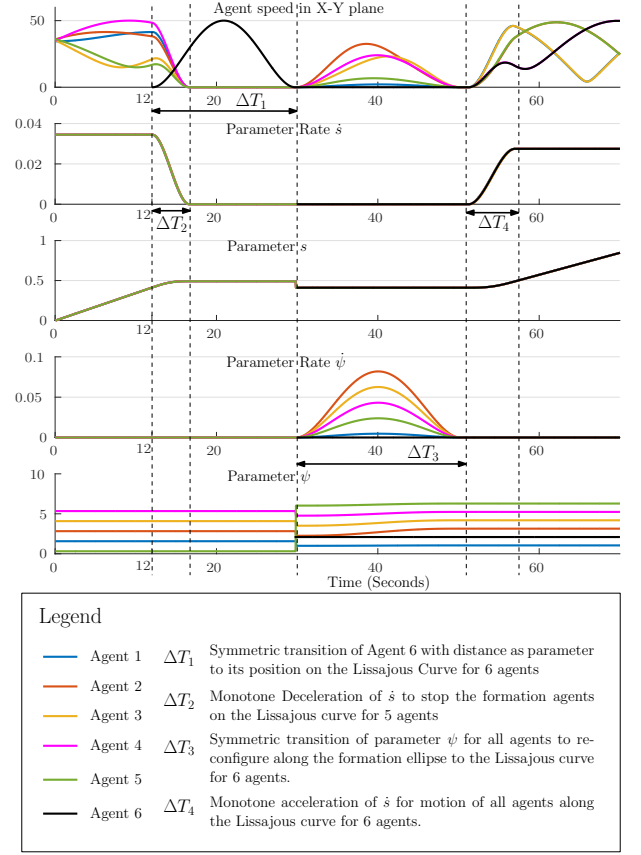}\\
\caption{Trajectories for agent addition case for MATLAB simulation 1}
\label{fig_sim_traj}
\end{figure}

In order to demonstrate the scalability for the second simulation, we assume a smaller sensor footprint radius and communication range  (refer to Table \ref{tab_algo_ip}). Thus the number of agents required for this case is larger. For both simulations, the circular hull radius of the agents is selected as the sufficient bound $r_{dm}$ given by Algorithm \ref{algo_initializer}. From the simulations, we observe that the motion of the agents both during surveillance along the Lissajous curve and during reconfiguration are collision-free, thus validating the strategies for surveillance and reconfiguration designed in Sections \ref{sec_surveillance} and \ref{sec_recon}.

\begin{remark} The videos of both the MATLAB\reg simulations for all three reconfiguration operations can be found at the web-link:\\ {\small {https://youtu.be/HEg5XfbBusY}}
\end{remark}


\subsection{Software-In-The-Loop simulation}\label{sec_sitl}
Since the proposed surveillance and reconfiguration strategies are developed for aerial agents such as quadrotors or helicopters, our final objective is to implement it  using programmable autonomous quadrotors.

The quadrotors built for the experiment use the Pixhawk v1\footnote{ \url{https://docs.px4.io/en/flight_controller/pixhawk.html}} flight controller running the PX4 flight stack\footnote{ \url{https://dev.px4.io/en/}} for stabilisation of the drone. Reference commands (such as commanded position) can be sent to the Pixhawk using the MAVLink communication packets\footnote{ \url{https://mavlink.io/en/}} on a serial channel.

 To simulate this setup, we use the Software-In-The-Loop (SITL) simulator\footnote{ \url{https://dev.px4.io/en/simulation/gazebo.html}} for the PX4 flight stack that uses the Robot Operating System (ROS) (Version: Kinetic Kame\footnote{ \url{http://wiki.ros.org/kinetic}})  along with the physics simulator Gazebo\footnote{ \url{http://gazebosim.org/}} (Version 8). This simulation is done on a computer system equipped with a NVIDIA Geforce GTX 1060 graphics card and running the Ubuntu 16.04 Xenial LTS operating system. For each quadrotor, an instance of the PX4 flight stack is  simulated, and the proposed strategy for surveillance (Sectiopn \ref{sec_surveillance}) and reconfiguration (Section \ref{sec_recon}) is implemented as a C++ script, which is written adhering to the node-topic structure of ROS (called {\it UAV\_i\_ctrl} node for agent $i$). The {\it MAVROS}\footnote{ \url{http://wiki.ros.org/mavros}} package is used to translate between the ROS interface and the MAVLink packets, which are sent via a UDP port to the PX4 flight stack simulation (SITL component) of the corresponding quadrotor. In this manner the SITL simulation environment allows simulation of the physics for multiple quadrotors along with a simulated instance of the PX4 flight stack for each agent. The SITL component  acquires the quadrotor  states as feedback and sends actuator command values to the Gazebo simulator as shown in the Fig.\ \ref{fig_sitl_blk}.

Each quadrotor receives communication data from all other quadrotors in the simulation. This is done for ease of implementing the multi-agent network in ROS, and the limited communication range $r_{com}$ is simulated in the navigation code for each agent by ignoring received data from agents outside a sphere of radius $r_{com}$ centered around the agent. We use the {\it joy}\footnote{ \url{http://wiki.ros.org/joy}} package to issue basic commands to the multi-agent formation such as  take-off, mission start,  land, and  reconfiguration commands, namely: \\ 1) removal command with agent ID\\ 2) replacement command with agent ID\\ 3) addition command to initiate formation reconfiguration.

\begin{figure}[!h]
\centering
\includegraphics[width=0.65\linewidth]{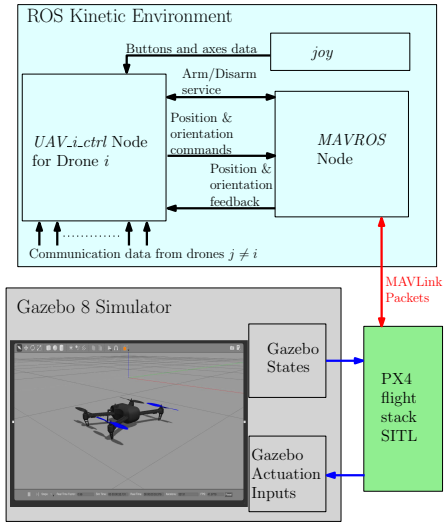}\\
\caption{SITL simulation framework for each 3DR Iris quadrotor in ROS-Gazebo environment, using the PX4 flight stack simulated in Loop}
\label{fig_sitl_blk}
\end{figure}

The inputs and outputs of  the Algorithm \ref{algo_initializer} for the SITL simulation are given in Tables \ref{tab_algo_ip} and \ref{tab_algo_op} respectively. From Table \ref{tab_algo_op}, we see that the upper bound on circular hull radius of the agents is $r_{dm}=0.459\ m$. For the SITL simulator, we have selected the 3DR Iris quadrotor\footnote{ \url{http://www.arducopter.co.uk/iris-quadcopter-uav.html}} model  (shown in Fig.\ \ref{fig_sitl_blk}) from the available models. This quadrotor has a motor to motor length of $55\ cm$ and a propeller diameter of $10\ inch$ (or $25.4\ cm$). Thus it has a circular hull of radius $r_d=0.402\ m<r_{dm}$ (calculated as  $\frac{55+24.5}{2}\ cm$). Gazebo simulates the complete physics of all the 3DR Iris quadrotors, along with  an instance of the simulated PX4 flight stack for each quadrotor, which tracks the reference position commands issued  using its internal PID control loops.

\begin{remark}\label{rem_sitl} The video of the SITL multi-quadrotor simulation can be found at the web-link: \\{\small  {https://youtu.be/XKXlvEDB-Qo}}\\ The top-left window is a video recording of the Gazebo SITL simulation, and the top-right window shows the video of the recorded simulation coordinates of the quadrotors,  plotted in MATLAB\reg.
\end{remark}

\subsection{Implementation}
\begin{figure}[h]
\centering
\includegraphics[width=0.9\linewidth]{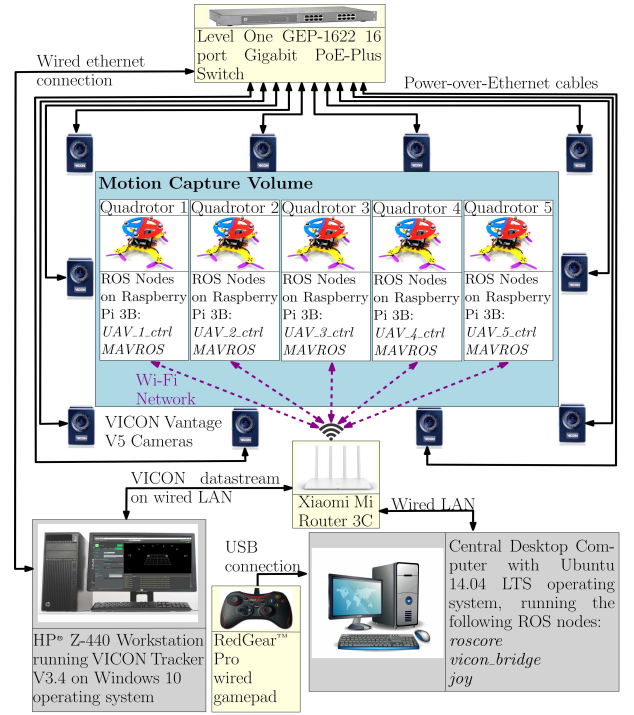}\\
\caption{Block diagram of the experimental setup}
\label{fig_multi_drone_expt_blk}
\end{figure}

\begin{figure}[h]
\centering
\includegraphics[width=0.8\linewidth]{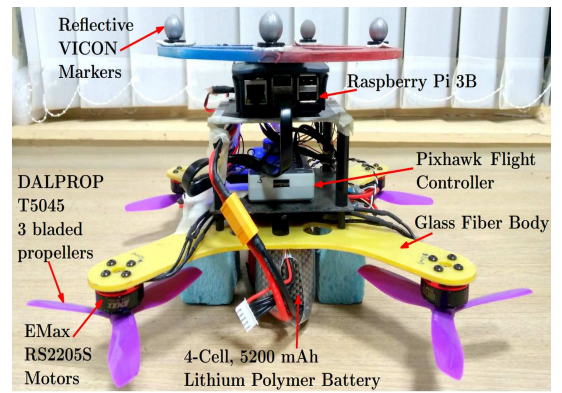}\\
\caption{Quadrotor used for experiment}
\label{dronepic}
\end{figure}

The experiment discussed here has been performed in an indoor laboratory space using a motion capture system for localisation. This system comprises of   ten Vantage V5\footnote{  \url{https://www.vicon.com/products/camerasystems/vantage}} cameras as shown in Fig.\ \ref{fig_multi_drone_expt_blk}. The  quadrotors used for the experiment were made in-house (shown in Fig.\ \ref{dronepic}). The design specifications of  the quadrotors are given in  Table \ref{tab_drone}. The sensor calibration and  inner PID control-loops tuning for these quadrotors for the PX4 flight stack running on the Pixhawk v1 flight controller was done using the QGroundControl\footnote{\url{http://qgroundcontrol.com/}} software.

\begin{table}[h]
\centering
\caption{Quadrotor Specifications}
\label{tab_drone}
\resizebox{7.0cm}{!}{\tiny
\begin{tabular}{|l|l|l|}
\hline
Structure                    & \multicolumn{2}{l|}{\begin{tabular}[c]{@{}l@{}}X-type frame $20\ cm \times 20\ cm$ \\ 3 mm thick glass fiber sheet\end{tabular}}                          \\ \hline
Motors                       & \multicolumn{2}{l|}{EMax RS2205S 2300kV}                                                                                                                  \\ \hline
Propellers                   & \multicolumn{2}{l|}{\begin{tabular}[c]{@{}l@{}}DALPROP T5045 3 bladed \\ 5 inch diameter, 4.5 inch pitch\end{tabular}}                                     \\ \hline
Speed Controllers & \multicolumn{2}{l|}{\begin{tabular}[c]{@{}l@{}}BLHeli\_S DShot600 \\ 30A Cicada ESC\end{tabular}}                                                         \\ \hline
Flight Controller            & \multicolumn{2}{l|}{\begin{tabular}[c]{@{}l@{}}Pixhawk v1 with Firmware: \\ px4fmu-v2\_lpe v1.6.0\\  release candidate 4\end{tabular}}                   \\ \hline
Companion Computer           & \multicolumn{2}{l|}{\begin{tabular}[c]{@{}l@{}}Raspberry Pi 3B with Ubuntu\\ Mate  16.04 Operating System and \\ ROS Kinetic Kame installed\end{tabular}} \\ \hline
Battery                      & \multicolumn{2}{l|}{\begin{tabular}[c]{@{}l@{}}4-Cell(4S), 25C, 5200mAh \\ Lithium-Polymer with \\ 16.8V peak voltage\end{tabular}}                       \\ \hline
Total Weight                 & \multicolumn{2}{l|}{1.18 kg}                                                                                                                              \\ \hline
Endurance                    & \multicolumn{2}{l|}{12-14 minutes}                                                                                                                        \\ \hline
\end{tabular}
}
\end{table}

The inputs and outputs of  the Algorithm \ref{algo_initializer} for the experiment are given in Tables \ref{tab_algo_ip} and \ref{tab_algo_op} respectively. From Table \ref{tab_algo_op} we see that the upper bound on circular hull radius of the agents is $r_{dm}=0.407\ m$. The quadrotor we have constructed has a motor to motor length of $28.3\ cm$ and a propeller diameter of $5\ inch$ (or $12.7\ cm$). Thus it has a circular hull of radius $r_d=0.205\ m<r_{dm}$ (calculated as  $\frac{28.3+12.7}{2}\ cm$). The same C++ script {\it UAV\_i\_ctrl} developed for the SITL simulation in Section \ref{sec_sitl} is used for implementing the proposed surveillance and reconfiguration strategy with the quadrotors shown in Fig.\ \ref{dronepic}. Each quadrotor has a Raspberry Pi 3B\footnote{\url{https://www.raspberrypi.org/products/raspberry-pi-3-model-b/}} companion computer onboard which runs the {\it UAV\_i\_ctrl} and {\it MAVROS} nodes for the corresponding quadrotor. As a result, the proposed strategy is implemented in a decentralized manner. A block diagram of the experimental setup is shown in Fig.\ \ref{fig_multi_drone_expt_blk}. The quadrotors are fitted with reflective  markers  for operation in the VICON motion capture space. The VICON cameras detect these markers and the HP workstation processes the camera data using the VICON Tracker V3.4 software\footnote{\url{https://docs.vicon.com/display/Tracker34/Vicon+Tracker+User+Guide}} and broadcasts its data stream on a Local Area Network.  The {\it roscore} node\footnote{\url {http://wiki.ros.org/roscore}} which is the master node for handling the complete ROS network  runs on a central computer  running ROS Indigo Igloo\footnote{ \url{http://wiki.ros.org/indigo}} on the Ubuntu 14.04 LTS. This computer processes the VICON data stream and converts it to a ROS compatible format using the {\it vicon\_bridge} node\footnote{ \url{http://wiki.ros.org/vicon_bridge}} in ROS Indigo. It also runs the {\it joy} node to read joystick commands for initiating take off, land, agent removal with ID, agent replacement with ID and agent addition. The {\it roscore} node running on the central computer interacts with the Raspberry Pi's via a Wi-Fi network setup using a wireless router. For this we use the concept of a  multicomputer ROS network\footnote{ \url{http://wiki.ros.org/ROS/Tutorials/MultipleMachines}}. Thus each Raspberry Pi receives localisation as well as inter-agent communication data via Wi-Fi within this multi-computer ROS network, and in turn commands the Pixhawk flight controller via a local instance of the {\it MAVROS} Node, using a USB wired link. The Raspberry Pi issues $(X,Y,Z)$ position commands corresponding to the $s^i,\psi^i$ parameters of agent $i$, and the flight altitude ($h_F=1.5\ m$ or $h_L=0.5\ m$). It also issues a constant heading command of $0\ rad$. The Pixhawk v1 flight controller tracks these commands using the PID loops in the PX4 firmware. These loops in turn use the on-board sensor data and localisation data available from the motion capture  system as feedback.

\begin{remark}\label{rem_expt} The video of the  multi-quadrotor experiment discussed above can be found at the web-link:\\ {\small {https://youtu.be/DUNR0-T9zTA}}\\ The top-left window is a video recording of the quadrotors in flight, and the top-right window shows the video of recorded position coordinates of the quadrotors as captured by the motion capture system (orthographic top view), plotted in MATLAB\reg. The bottom left video shows the motion capture markers on the quadrotors  seen by the VICON cameras on the VICON Tracker V3.4 software.
\end{remark}

From the videos in Remark \ref{rem_sitl} and Remark \ref{rem_expt}, we see that for the  SITL simulation and the experiment in a VICON environment, all three reconfiguration operations are performed with smooth collision-free trajectories of the quadrotors in the same simulation/flight, hence validating our proposed multi-agent surveillance and formation reconfiguration strategy. The  step commands are only used for initialisation of the quadrotor  positions in the formation, and for changing altitude from $h_F$ to $h_L$ and vice-versa.

\section{Conclusion }
\label{sec_conc}\vspace{-0.2cm}
We have proposed in \cite{lisiros}, a multi-agent formation on  Lissajous curves which performs collision-free surveillance of a rectangular area. We have  proposed here a reconfiguration strategy whereby a quadrotor can be added, removed or replaced using a decentralized cooperating scheme. 

We have validated our results through MATLAB\reg simulations for agents having a non-zero size satisfying a theoretically derived size bound. To demonstrate the practical applicability of the proposed surveillance and reconfiguration strategies, we have also presented simulations, for quadrotors in a ROS-Gazebo based Software-In-The-Loop simulator and have implemented the same with a team of five quadrotors in a motion capture environment. This work has potential applications in security, asset protection, agricultural monitoring, distributed sensing, etc.

\section*{Acknowledgements \vspace{-0.2cm}}
We  thank  Vraj Parikh and  Shoeb Ahmed Adeel for their help with the design and construction of the quadrotors.  We also thank Gaurav Gardi for assistance with setting up the SITL simulator, and Mr.\ Kumar Khot for help with logistics and setting up of the VICON motion capture system. The quadrotors were constructed and flight tested in the Miniature Aerial Vehicle (MAV) Laboratory at the Department of Aerospace Engineering, IIT Bombay, and the experiments were conducted in the Autonomous Robots and Multi-agent Systems (ARMS) Laboratory at the Interdisciplinary Progamme  in Systems and Control Engineering, IIT Bombay.

\section*{Appendix}

\noindent{\bf Proof of  Lemma \ref{clm_calculus_variations}:}
Let $V(\dddot{g}(t),t)=\frac{{\dddot{g}}^2(t)}{2}$, then cost $J=\int\limits_{0}^{T_f} V(\dddot{g}(t),t) dt $. For convenience of notation denoting $k^{th}$ detivative of $g(t)$ as $g^{(k)}(t)$,  the first variation $\partial J$ is computed as:
\begin{align}
\partial J= \int\limits_{0}^{T_f}\left( V\left(\dddot{g}(t)+\delta_{g^{(3)}}(t),t\right) -   V\left(\dddot{g}(t),t\right)\right) \ dt
\label{eqn_f_variations}
\end{align}
and by the necessary condition for optimality, at the optimal trajectory $g^*(t)$ (denoted as $g^*_t$ for short),   $\partial J=0$. Thus using Taylor series expansion in \eqref{eqn_f_variations} about $g(t)=g^*_t$ yields
$
\partial J= \int\limits_{0}^{T_f} \left.  \frac{\partial V(g^{(3)}(t),t)}{\partial g^{(3)}(t)}\right\vert_{g^*_t} \delta_{g^{(3)}}(t) \ dt.
$
By repeated application of integration by parts,
{\small
\begin{align*}
\partial J&= \left.\frac{\partial V(g^{(3)}(t),t)}{\partial g^{(3)}(t)}\right\vert_{g^*_t}\left. \delta_{g^{(2)}}(t)\right\vert_{0}^{T_f}-
\frac{d}{dt}\left( \left.\frac{\partial V(g^{(3)}(t),t)}{\partial g^{(3)}(t)}\right\vert_{g^*_t}\right)\left. \delta_{g^{(1)}}(t)\right\vert_{0}^{T_f} \nonumber \\
&+ \frac{d^2}{dt^2}\left(\left.\frac{\partial V(g^{(3)}(t),t)}{\partial g^{(3)}(t)}\right\vert_{g^*_t}\right)\left. \delta_{g^{{}}}(t)\right\vert_{0}^{T_f}
-\int\limits_{0}^{T_f} \frac{d^3}{dt^3}\left(\left.  \frac{\partial V(g^{(3)}(t),t)}{\partial g^{(3)}(t)}\right\vert_{g^*_t}\right) \delta_{g}(t) \ dt.
\end{align*}}
From the fixed boundary conditions, $\delta_{g^{(2)}}(T_f)=\delta_{g^{(2)}}(0)=0$, $\delta_{g^{(1)}}(T_f)=\delta_{g^{(1)}}(0)=0$, $\delta_{g}(0)=0$. Thus, the first variation simplifies to
{\small
\begin{align}
\partial J&=  \frac{d^2}{dt^2}\left(\left.\frac{\partial V(g^{(3)}(t),t)}{\partial g^{(3)}(t)}\right\vert_{g^*_t}\right) \delta_{g}(T_f)
-\int\limits_{0}^{T_f} \frac{d^3}{dt^3}\left(\left.  \frac{\partial V(g^{(3)}(t),t)}{\partial g^{(3)}(t)}\right\vert_{g^*_t}\right) \delta_{g}(t) \ dt.
\label{eqn_f_variation}
\end{align}}
 By substituting $V(\dddot{g}(t),t)=\frac{{\dddot{g}}^2(t)}{2}$, the Euler-Lagrange equation \eqref{eqn_f_variation} simplifies to $g^{*(6)}(t)=0$. Then from the Euler-Lagrange equation, for some constants $c_1,c_2,c_3,c_4,c_5$ and $c_6$,
 \begin{align}
 g^*(t)&=c_1 t^5+c_2 t^4+c_3 t^3+c_4 t^2+c_5 t +c_6,  \label{eqn_g_c}\\
  g^{*(1)}(t)&=5c_1 t^4+4c_2 t^3+3c_3 t^2+2c_4 t+c_5, \label{eqn_g1_c}\\
    g^{*(2)}(t)&=20c_1 t^3+12c_2 t^2+6c_3 t+2c_4, \label{eqn_g2_c}\\
      g^{*(3)}(t)&=60c_1 t^2+24c_2 t+6c_3 , \label{eqn_g3_c}\\
      g^{*(4)}(t)&=120c_1 t+24c_2, \label{eqn_g4_c}\\
      g^{*(5)}(t)&=120c_1. \label{eqn_g5_c}
 \end{align}

The solutions for both the boundary value conditions \ref{bcond_c1} and \ref{bcond_c2} are as follows:

\noindent 1. Since $\partial J=0$ and $\delta_{g}(T_f)$ need not be zero, $\frac{d^2}{dt^2}\left(\left.\frac{\partial V(g^{(3)}(t),t)}{\partial g^{(3)}(t)}\right\vert_{g^*(t)}\right)=0$ in addition to the Euler-Lagrange equation. This simplifies to $g^{*(5)}(t)=0$ and   from \eqref{eqn_g5_c}, $c_1=0$. From the given boundary conditions at $t=0$ and \eqref{eqn_g_c}, \eqref{eqn_g1_c}, \eqref{eqn_g2_c} we get $c_6=g_0$, $c_5=\dot{g}_0$ and $c_4=0$ respectively. The boundary conditions at $t=T_f$ gives $0=12c_2T_f^2+6c_3T_f$ and $\dot{g}_f=4c_2T_f^3+3c_3T_f^2+\dot{g}_0$ which gives rise to this linear system of equations
\begin{align}
\left[\begin{matrix}
2T_f & 1\\
4T_f^3 & 3T_f^2
\end{matrix}\right]\left[\begin{matrix}
c_2\\ c_3
\end{matrix}\right]=\left[\begin{matrix}
0\\ \dot{g}_f-\dot{g}_0
\end{matrix}\right].
\end{align}
Solving this system of linear equations,
$c_2=-\frac{\dot{g}_f-\dot{g}_0}{2T_f^3},\ c_3=\frac{\dot{g}_f-\dot{g}_0}{T_f^2}$. Thus,
$g^*(t)=(\dot{g}_f-\dot{g}_0)\left(-\frac{t^4}{2T_f^3}+\frac{t^3}{T_f^2}\right)+\dot{g}_0t+g_0. 
$\\

\noindent 2. The second set of boundary conditions \ref{bcond_c2} represent a fixed end-time $T_f$ and fixed end state $g(T_f)=g_f$ problem. Thus at $t=0$ and $t=T_f$, \eqref{eqn_g_c}- \eqref{eqn_g5_c} result in the following Linear system of equations.
\begin{align}
\left[\begin{matrix}
 T_f^5  & T_f^4 &  T_f^3  &  T_f^2   &  T_f &  1\\
     0    &     0   &    0    &   0   &  0 &  1 \\
    5T_f^4  & 4T_f^3  &  3T_f^2  &   2T_f  &  1 &  0\\
      0    &   0    &    0    &   0   & 1 &  0\\
    20T_f^3 & 12T_f^2 &  6T_f    &    2  &  0 &  0\\
      0   &    0    &    0    &    2  &  0 &  0
\end{matrix}\right]\left[\begin{matrix}
c_1\\c_2\\ c_3\\ c_4\\ c_5\\c_6
\end{matrix}\right]=\left[\begin{matrix}
g_f\\g_0\\ \dot{g}_f\\ \dot{g}_0\\ \ddot{g}_f\\\ddot{g}_0
\end{matrix}\right]=\left[\begin{matrix}
g_f\\g_0\\0\\ 0\\ 0\\0
\end{matrix}\right].
\label{eqn_gf_Tf_system}
\end{align}
 Solving \eqref{eqn_gf_Tf_system}, gives \\
 $c_1=\frac{6(g_f-g_0)}{T_f^5},\ c_2= \frac{15(g_0-g_f)}{T_f^4},\ c_3=\frac{10(g_f-g_0)}{T_f^3},$
 $ c_4=c_5=0,\ c_6=g_0$

Thus,
$g^*(t)=(g_f-g_0)\left(10T_f^2-15T_ft+6t^2\right)\frac{t^3}{T_f^5}+g_0.
$\qed

\begin{lemma}
\label{clm_parabolic_lb}
For $u\in[\frac{-\pi}{2},\ \frac{\pi}{2}]$, $\sin^2 u \geq \frac{4}{\pi^2} u^2$
\end{lemma}
\begin{proof}
Consider the following cases:\\
\noindent\textit{Case 1: }For $u=0,\pm \frac{\pi}{2}$, $\sin^2 u = \frac{4}{\pi^2}u^2$. Thus the claim holds true at these values.\\
\noindent\textit{Case 2: }For $u\in \left(\frac{-\pi}{2},\ 0\right)\cup \left(0,\ \frac{-\pi}{2}\right) $, let $f(u)=\frac{\sin^2 u}{u^2}$. By differentiating and simplifying the resultant expression, $f'(u)=\frac{\sin 2u}{u^3}(u-\tan u)$. Using  the infinite Taylor series for $\tan u=u+\frac{u^3}{3}+\frac{2u^5}{15}+...$ for $\vert u\vert < \frac{\pi}{2}$, $f'(u)=-\sin 2u\left(\frac{1}{3}+\frac{2u^2}{15}+...\right) \mbox{ for } \vert u\vert < \frac{\pi}{2}$.

Thus for $u\in\left(\frac{-\pi}{2},\ 0\right)$, $f'(u)>0$ which implies that $f(u)$ is monotone increasing and thus $f(u)>f\left(\frac{-\pi}{2}\right)$ which implies $\frac{\sin^2 u}{u^2}>\frac{4}{\pi^2}$ or $\sin^2 u>\frac{4}{\pi^2}u^2$.
Similarly for $u\in\left(0,\ \frac{\pi}{2}\right)$, $f'(u)<0$ which implies that $f(u)$ is monotone decreasing and thus $f(u)>f\left(\frac{\pi}{2}\right)$ which implies $\frac{\sin^2 u}{u^2}>\frac{4}{\pi^2}$ or $\sin^2 u>\frac{4}{\pi^2}u^2$. Hence, \textit{Case 1} and \textit{Case 2} considered together prove the claim.
\end{proof}

\noindent{\bf Proof of Proposition  \ref{prop_ellipse_avoid}:}
For any two agents $i,j$ in the formation, $\psi_c^j=\psi_c^i+\Delta^{ij}\psi \bmod{2\pi}$. From \eqref{eqn_dist_ij}, the Euclidean distance between agents $i$ and $j$ is given by $D_{ij}=2\sin(\frac{\Delta^{ij}\psi}{2})\sqrt{A^2\sin^2\left(\Psi_p-a_cs_c\right)+B^2\cos^2\left(\Psi_p+b_cs_c\right)}$ and $\Psi_p =\psi_c^i+\frac{\Delta^{ij}\psi}{2}$. From  Lemma \ref{clm_parabolic_lb} in the appendix, we get a parabolic lower bound for $A^2\sin^2(\psi_c^i+\frac{\Delta^{ij}\psi}{2}-a_cs_c)$ and $B^2\cos^2(\psi_c^i+\frac{\Delta^{ij}\psi}{2}+b_cs_c)$ as $\frac{4}{\pi^2}A^2u^2$ and $\frac{4}{\pi^2}B^2\left(u+(a_c+b_c)s_c- \frac{(2k-1)\pi}{2}\right)^2$ respectively, where $u=\psi_c^i+\frac{\Delta^{ij}\psi}{2}-a_cs_c$ for some $k\in{\Bbb N}$. Thus $D_{ij}^2>4\sin^2(\frac{\Delta^{ij}\psi}{2})f(u)$, where
$$
f(u)=\frac{4}{\pi^2}A^2u^2+\frac{4}{\pi^2}B^2\left(u+N_cs_c- \frac{(2k-1)\pi}{2}\right)^2.
$$
Solving $\frac{df(u)}{du}=0$ we get
$u^*=-\frac{B^2}{A^2+B^2}\left(N_cs_c- \frac{(2k-1)\pi}{2}\right)$ which is a minimizer as $f(u)$ is a sum of parabolic functions having positive coefficients. Thus if $D_{ij_{min}}$ is the minimum distance between the agents $i$ and $j$, then $D_{{ij}_{min}}^2\geq 4\sin^2(\frac{\Delta^{ij}{\psi}}{2})f(u^*)$. Substituting value of $u^*$  in this inequality and simplifying the resultant expression leads to
\begin{align}
D_{{ij}_{min}}\geq \frac{4}{\pi}\frac{ABN_c}{\sqrt{A^2+B^2}}\left\vert\sin\left(\frac{\Delta^{ij}\psi}{2}\right)\right\vert\Delta_s,
\label{eqn_dij_min_psi}
\end{align}
where $\Delta_s=\left\vert s_c- \frac{(2k-1)\pi}{2N_c}\right\vert$. Hence if we guarantee
\begin{align}\frac{4}{\pi}\frac{ABN_c}{\sqrt{A^2+B^2}}\left\vert\sin\left(\frac{\Delta^{ij}_{min}}{2}\right)\right\vert\Delta_s>2r_{dm}
\label{eqn_ellipse_col_free}
\end{align} where $\Delta^{ij}_{min}=\min_{t\in\Bbb {R}^+} \vert\Delta^{ij}\psi(t)\vert$                                                                                                                                                        , then we can ensure $D_{ij_{min}}> 2r_{dm}$ and the transition along the ellipse is guaranteed to be collision-free.
   Rearranging \eqref{eqn_ellipse_col_free} results in the following inequalities:
$s_c> s_{diag}(k)+ \delta_s$ and  $s_c< s_{diag}(k)- \delta_s $,
 where $s_{diag}(k)=\frac{(2k-1)\pi}{2N_c}$ for $k\in \Bbb N$ is the parameter value for which the agents lie on the the degenerate ellipse (or the diagonal line)  as shown in Fig.\ \ref{fig_lis_ellipse}, and $\delta_s=\frac{\pi}{2N_c}\frac{r_{dm}\sqrt{A^2+B^2}}{AB}\left\vert\sin\left(\frac{\Delta^{ij}_{min}}{2}\right)\right\vert^{-1}$.
Thus all ellipses corresponding to all the $s$ values satisfying
\begin{align}
s_c\not\in \cup_{k\in\Bbb N}\left(s_{diag}(k)-\delta_s ,\ s_{diag}(k)+\delta_s \right)\bmod{2\pi}
\label{eqn_s_avoid}
\end{align} are feasible for agent transitions.\qed
\\\\
\noindent{\bf Proof of Lemma  \ref{lem_ellipse_partitions}:} If $Q_1, \cdots, Q_{N+1}$ are $N + 1$ equi-spaced points on $C$, by pigeonhole principle, at least one of the segments $[P_i,\ P_{i+1})$ contains two $Q_j$'s. This proves the existence of at least one such  $Q_j, Q_{j+1}$ pair within a $[P_i,\ P_{i+1})$ segment (with $j\in\{1,...,N+1\}$).

  There are $N$  equi-spaced segments  of the form $[P_{i},\  P_{i+1})$ which are disjoint and cover $C$, so if two of them contain a pair of $Q_j$'s each, the remaining $N-3$ of $Q_j$'s are distributed among the remaining $N-2$  segments, thus at least one of them does not contain any $Q_j$ points. Suppose the segment $[P_{i'}\ P_{i'+1})$  for some $i' \in \{1,...,N\}$, doesn't contain a $Q_j$, then this segment in turn is contained within a $[Q_{j'},\ Q_{j'+1})$ segment for some $j'\in\{1,...,N+1\}$. But if the parametric curve length of $C$ is $L$, then $[P_{i'}\ P_{i'+1})$ is of length $\frac{L}{N}$ and $[Q_{j'},\ Q_{j'+1})$ is of length $\frac{L}{N+1}$ and $[P_{i'}\ P_{i'+1})\not\subset [Q_{j'},\ Q_{j'+1})$, which is a contradiction. Thus there is exactly one $Q_j$ pair contained in a  $[P_i,\ P_{i+1})$ segment.\qed
\bibliographystyle{elsarticle-harv}
\bibliography{References}
\end{document}